%% file: main.tex
\documentclass{article}

\PassOptionsToPackage{numbers, compress}{natbib}
\usepackage[dvipsnames]{xcolor}

\usepackage[utf8]{inputenc} %
\usepackage[T1]{fontenc}    %
\usepackage[allcolors=NavyBlue,colorlinks=true,backref=page]{hyperref}       %
\usepackage{url}            %
\usepackage{booktabs}       %
\usepackage{amsfonts}       %
\usepackage{nicefrac}       %
\usepackage{microtype}      %
\usepackage{xspace}
\usepackage{subcaption}

\usepackage{wrapfig}
\usepackage{flushend}

\definecolor{darkblue}{RGB}{0, 0, 139}

\hypersetup{
    colorlinks=true,
    linkcolor=black,
    citecolor=darkblue, %
    urlcolor=darkblue
}

\usepackage{algorithm}
\usepackage[noend]{algpseudocode}

\usepackage[most]{tcolorbox}
\definecolor{colorcomment}{RGB}{160, 190, 210}%
\makeatletter
\algnewcommand{\LineComment}[1]{\Statex \hskip\ALG@thistlm \(\triangleright\) 
{\color{colorcomment}#1}}
\makeatother

\makeatletter
\algnewcommand{\IndentLineComment}[1]{\Statex \hskip\ALG@tlm \(\triangleright\) {\color{colorcomment}#1}}
\makeatother

\usepackage[accepted]{icml2023}

\usepackage{cite}
\usepackage{comment}

\usepackage{amsfonts}
\usepackage{makecell}
\usepackage{adjustbox}
\usepackage{multicol}

\usepackage{amsmath}
\usepackage{amssymb}
\usepackage{mathtools}
\usepackage{amsthm}
\usepackage{bbm}

\usepackage[textsize=footnotesize,textwidth=2.0cm]{todonotes}

\usepackage{xspace}

\graphicspath{{../}}

\usepackage{thmtools} 
\usepackage{thm-restate}

\input{macros}

\icmltitlerunning{Active Policy Improvement from Multiple Black-box Oracles}

\begin{document}

\twocolumn[
\icmltitle{Active Policy Improvement from Multiple Black-box Oracles} 

\icmlsetsymbol{equal}{*}

\begin{icmlauthorlist}
\icmlauthor{Xuefeng Liu}{equal,uchicago}
\icmlauthor{Takuma Yoneda}{equal,ttic}
\icmlauthor{Chaoqi Wang}{equal,uchicago}
\icmlauthor{Matthew R.\ Walter}{ttic}
\icmlauthor{Yuxin Chen}{uchicago}
\end{icmlauthorlist}

\icmlaffiliation{uchicago}{Department of Computer Science, University of Chicago, Chicago, IL, USA}
\icmlaffiliation{ttic}{Toyota Technological Institute at Chicago, Chicago, IL, USA}

\icmlcorrespondingauthor{Xuefeng Liu}{xuefeng@uchicago.edu}

\vskip 0.3in
]

\printAffiliationsAndNotice{\icmlEqualContribution}
\input{abstract}

\input{introduction}

\input{related_works}

\input{background_problemStatement}

\input{algorithm}

\input{analysis}

\input{experiments}

\input{conclusion}

\bibliography{reference}
\bibliographystyle{icml2023}

\flushcolsend

\onecolumn

\appendix

\vfill

\clearpage

\input{supp_notation}
\input{supp_algo_characteristics}
\input{supp_theory}

\input{supp_theory_proofs}

\input{supp_additional_experiments}

\end{document}

%% file: macros.tex
\newcommand\loss{\ell}

\newcommand\Policies{\ensuremath{\Pi}}
\newcommand\policy{\ensuremath{\pi}}

\newcommand\state{s}

\newcommand\stateDist{d}

\newcommand\transDynamics{\mathcal{P}}
\newcommand\reward{r}
\newcommand\RFunc{r}

\newcommand\VFunc{\ensuremath{\ensuremath{V}}}
\newcommand\Func{\ensuremath{\ensuremath{f}}}
\newcommand\QFunc{\ensuremath{\ensuremath{Q}}}
\newcommand\Adv{{A}}
\newcommand\AFunc{{A}}

\newcommand\discFactor{\gamma}

\newcommand\ONum{\ensuremath{K}}

\newcommand\stateSpace{\ensuremath{\mathcal{S}}}
\newcommand\action{\ensuremath{a}}

\newcommand\actionSpace{\ensuremath{\mathcal{A}}}
\newcommand\trajectory{\ensuremath{\mathcal{\tau}}}
\newcommand\horizon{\ensuremath{H}}
\newcommand\ENum{\ensuremath{N}} %

\newcommand\threshold{\ensuremath{{\Gamma_{\state}}}}

\newcommand\dataset{\ensuremath{\mathcal{D}}}

\newcommand\RewardFunc{\ensuremath{r}}

\newcommand\MDP{\ensuremath{\mathcal{M}}}
\newcommand\var{\ensuremath{{v}}}

\renewcommand{\Pr}[1]{\ensuremath{\mathbb{P}\left[#1\right] }}

\newcommand{\expct}[1]{\mathbb{E}\left[#1\right]}
\newcommand{\expctover}[2]{\mathbb{E}_{#1}\!\left[#2\right]}

\def \argmax {\mathop{\rm arg\,max}}

\newcommand{\algname}{\textsc{MAPS}\xspace}
\newcommand{\lopsaps}{\textsc{MAPS}\xspace}

\newcommand{\lopsase}{\textsc{MAPS-SE}\xspace}

\newcommand{\mamba}{\textsc{MAMBA}\xspace}

\newif\iffinal

\iffinal
    \newcommand{\fix}[1]{}
    \newcommand{\YC}[1]{}
    \newcommand{\YCinline}[1]{}
    \newcommand{\XLinline}[1]{}
    \newcommand{\MW}[1]{}
    \newcommand{\MWinline}[1]{}
    \newcommand{\TY}[1]{}
    \newcommand{\TYinline}[1]{}
    \newcommand{\note}[1]{}
    \newcommand{\pref}[1]{}
\else
    \newcommand{\fix}[1]{{#1}}
    \newcommand{\YC}[1]{\todo[fancyline,color=NavyBlue!40]{YC: #1}\xspace}
    \newcommand{\YCinline}[1]{\textcolor{NavyBlue}{[YC: #1]}}
    \newcommand{\MW}[1]{\todo[fancyline,color=Red!40]{MW: #1}\xspace}
    \newcommand{\MWinline}[1]{\textcolor{Red}{[MW: #1]}}
    
    \newcommand{\XLinline}[1]{\textcolor{Maroon}{[XL: #1]}}
    \newcommand{\TYinline}[1]{\textcolor{Green}{[TY: #1]}}
    \newcommand{\TY}[1]{\todo[fancyline,color=Green!40]{TY: #1}\xspace}
    
    \newcommand{\note}[1]{{\color{purple}[XL: #1]}}
    
    \newcommand{\pref}[1]{{\color{blue}(\ref{#1})}}
\fi

\newcommand{\tabref}[1]{Table~\ref{#1}}
\newcommand{\figref}[1]{Fig.~\ref{#1}}

\newcommand{\appref}[1]{Appendix~\ref{#1}}
\newcommand{\thmref}[1]{Theorem~\ref{#1}}

\newcommand{\lemref}[1]{Lemma~\ref{#1}}

\newcommand{\algoref}[1]{Algorithm~\ref{#1}}
\newcommand{\linref}[1]{Line~\ref{#1}}

\newcommand{\paren} [1] {\ensuremath{ \left( {#1} \right) }}

\newcommand{\bracket}[1]{\left[#1\right]}

\newcommand{\curlybracket}[1]{\ensuremath{\left\{#1\right\}}}

\theoremstyle{plain}
\newtheorem{theorem}{Theorem}[section]

\newtheorem{lemma}[theorem]{Lemma}

\theoremstyle{definition}

\theoremstyle{remark}

%% file: abstract.tex
\begin{abstract}

Reinforcement learning (RL) has made significant strides in various complex domains. However, identifying an effective policy via RL often necessitates extensive exploration. Imitation learning aims to mitigate this issue by using expert demonstrations to guide exploration. 
In real-world scenarios, one often has access to multiple \emph{suboptimal} black-box experts, rather than a single optimal oracle. 
These experts do not universally outperform each other across all states, presenting a challenge in actively deciding which oracle to use and in which state. 
{We introduce %
\lopsaps and \lopsase, a class of policy improvement algorithms that perform imitation learning from multiple suboptimal oracles. In particular, \lopsaps actively
selects which of the oracles to imitate and improve their value function estimates, and \lopsase additionally leverages an active state exploration criterion to determine which states one should explore.}
We provide a comprehensive theoretical analysis 
and demonstrate that %
\lopsaps and \lopsase enjoy sample efficiency advantage over the state-of-the-art %
policy improvement algorithms. Empirical results show that \lopsase significantly accelerates policy optimization via state-wise imitation learning from multiple oracles across a broad spectrum of
control tasks in the DeepMind Control Suite. 
\end{abstract}

%% file: introduction.tex
\section{Introduction}

Reinforcement learning (RL) has achieved exceptional performance in many domains, such as robotics~\citep{kober2011reinforcement,kober2013reinforcement}, video games~\citep{mnih2013playing}, and Go~\citep{silver2017mastering}. However, RL tends to be highly sample inefficient in high-dimensional environments due to the need for exploration~\citep{sutton2018reinforcement}.
The sample complexity of RL limits its application to many real-world domains for which interactions with the environment can be costly.
To address this challenge, imitation learning (IL) has emerged as a promising alternative. IL improves the sample efficiency of RL by training a policy to imitate the actions of an expert policy, which is typically assumed to be optimal or near-optimal. This allows the agent to learn from expert demonstrations, reducing the need for costly trial-and-error exploration.

Previous works such as behavioral cloning~\citep{pomerleau1988alvinn}, DAgger~\citep{ross2011reduction}, and AggreVaTe(D)~\citep{ross2014reinforcement, sun2017deeply} demonstrate the effectiveness of IL with a single, near-optimal oracle.
In real-world settings, however, access to an optimal oracle\footnote{We use the terms ``expert'' and ``oracle'' interchangeably.} may be prohibitively expensive or they may simply be unavailable. Instead, it is often the case that the learner has access to \emph{multiple} suboptimal oracles~\citep{cheng2020policy} that are presented to the learner as black boxes, such as in the case of autonomous driving~\citep{lee2020mixgail}, 
robotics~\citep{yang2020multi}, and 
medical diagnosis~\citep{le2023learning}.
A naive approach to IL in these settings---choosing one of the oracles to imitate---can result in suboptimal performance, particularly when the relative utility of the oracles differs according to the state.

These considerations emphasize the importance of attaining sample efficiency when learning from black-box oracles by harnessing their state-wise expertise. Although practically relevant, this problem remains largely unexplored. Inspired by prior works~\citep{cheng2020policy,liu2022cost,liucontextual}, 
{we propose %
\lopsase~(\textbf{M}ax-aggregation \textbf{A}ctive \textbf{P}olicy \textbf{S}election with Active \textbf{S}tate \textbf{E}xploration), an algorithm that actively learns from multiple suboptimal oracles by exploiting their complementary utilities to train a policy that aims to surpass each oracle in a sample-efficient manner.
\lopsase
actively selects the oracle to roll out and uses the resulting trajectory to improve its value function estimate.
Furthermore, %
it is integrated with an active state exploration criterion that actively selects a state to explore based on the uncertainty of its state value.}

{To offer a better understanding of the proposed algorithm, we first investigate a special setting that does not include the active state exploration component. We refer to this algorithm as \lopsaps. \lopsaps generalizes MAMBA~\citep{cheng2020policy}---the current state-of-the-art (SOTA) approach to learning from multiple oracles---by a novel utilization of the active policy selection component. We then prove improvements of \lopsaps over MAMBA in terms of the sample complexity. }
Additionally, we pinpoint issues with MAMBA, such as the uncontrolled switching of roll-in (learner policy) and roll-out (expert policy)~(RIRO) and approximation errors of the gradient estimates. {We then provide a theoretical analysis that demonstrates the advantages of the active state exploration component in \lopsase.}%

Lastly, we conduct extensive experiments on the DeepMind Control Suite benchmark that compare \algname with MAMBA~\citep{cheng2020policy}, PPO~\citep{schulman2017proximal} with GAE~\citep{schulman2015high}, and the best oracle from the oracle set. We empirically show that \algname outperforms the current state-of-the-art (MAMBA). We present an analysis of these performance gains as well as the sample efficiency of our algorithm, which aligns with our theory. We then evaluate our proposed \algname-SE algorithm and demonstrate that it provides performance gains over \algname through various experiments. %

%% file: related_works.tex
\section{Related Work}\label{sec:related}

In this section, we provide a review of existing literature on %
model selection and imitation learning from one or multiple oracles. We refer the reader to Table~\ref{table:alg_characteristics} of Appendix~\ref{app:problem_setup} for a detailed comparison with prior art. %

\subsection{Policy / Model Selection}

Policy selection, commonly referred to as \emph{model selection} \citep{yang2022offline}, {considers the problem of selecting from or ranking a given set of policies} and arises in various aspects of reinforcement learning. These include the selection of a state abstraction~\citep{jiang2017theory, jiang2015abstraction}, the choice of a learning algorithm, and the process of feature selection~\citep{foster2019model, pacchiano2020model, liucontextual}.
The significance of policy selection has garnered increased attention due to its practical implications~\citep{paine2020hyperparameter, fu2021benchmarks}. 

Recent advancements in this field include active
offline policy selection~(A-OPS)~\citep{konyushova2021active}, which leverages policy similarities to improve value predictions and emphasizes an active setting in which a small number of environment evaluations can enhance the prediction of the optimal policy. Existing policy selection methods have several limitations, such as assuming knowledge of an optimal policy that can be queried for each state, the inability to incorporate a learner policy as part of the selection set, being restricted to a state-less online learning setting, or poor sample efficiency. In this work, we aim to overcome these limitations by actively approximating the value function of multiple blackbox oracles and learning a state-dependent learner policy in a manner that achieves a better sample efficiency than 
the state-of-the-art method.

\subsection{Active / Interactive Imitation Learning}

Offline imitation learning (IL) methods, such as behavioral
cloning, require an offline dataset of trajectories collected from one or more experts, which can lead to cascading errors in the learner
policy. Interactive IL methods, such as DAgger~\citep{ross2011reduction} and AggreVaTe~\citep{ross2014reinforcement}, assume that the learner can actively request a demonstration starting from the current state. With DAgger, the learner aims to imitate the oracle, regardless of its quality. AggreVaTe and its policy gradient variant AggreVaTeD~\citep{sun2017deeply} improve upon DAgger by incorporating the cost-to-go into policy training to prevent the learner from blindly following potentially unreasonable actions suggested by the oracle. However, these interactive IL approaches do not consider the cost of querying the oracle. Active imitation learning methods, on the other hand, reason over the utility of requesting a demonstration~\citep{shon2007active, judah2012active}. Among them, LEAQI~\citep{brantley2020active} uses a heuristic algorithm to actively decide when not to query experts to reduce interaction costs, but it focuses on a single-oracle setting. Our work lies at the intersection of active/interactive IL and RL, allowing our method to handle both single- and multiple-oracle settings, with a focus on efficiently learning from multiple oracles.

\subsection{Learning from Multiple Oracles}

Several methods, including EXP3 and EXP4~\citep{auer2002nonstochastic}, EXP4.P~\citep{beygelzimer2011contextual}, and Hedge~\citep{freund1997decision}, frame the problem of learning from multiple oracles as a contextual bandit or online learning problem. 
{Similarly, CAMS~\citep{liu2022cost,liucontextual} considers active learning from multiple experts, but only in an online learning setting {(with full observations of the oracles' losses and no state transitions)}.} However, these methods cannot handle sequential decision-making problems such as Markov decision processes (MDPs) due to their inability to incorporate state information. 
In RL and IL, ILEED~\citep{ilbelaiev} differentiates between oracles based on their expertise at each state, but is limited to pure offline IL settings. MAMBA~\citep{cheng2020policy} considers imitation learning from multiple experts, but it selects an oracle to query at random, compromising its sample efficiency.%
In contrast, our work improves sample efficiency in a multi-oracle setting while reasoning over their state-dependent relative performance by implementing active policy selection and active state exploration.

%% file: background_problemStatement.tex
\section{Preliminaries}

In this work, we focus on a finite-horizon Markov decision process (MDP) denoted as $\MDP_0=\langle\stateSpace,\actionSpace,\transDynamics,\RFunc, H\rangle$, where $\stateSpace$ represents the state space, $\actionSpace$ represents the action space, $\transDynamics: \stateSpace \times \actionSpace \rightarrow \Delta(\stateSpace)$ denotes the unknown state transition function with $\Delta(\stateSpace)$ being the set of distributions over set $\stateSpace$, $\RFunc: \stateSpace \times \actionSpace \rightarrow [0,1]$ denotes the unknown reward function, and $H$ represents the horizon of the episode. The policy $\policy: \stateSpace \rightarrow  \Delta(\actionSpace)$ maps the current state to a distribution over actions. We define the set of $\ONum$  oracles as $\Policies=\{\policy^k\}_{k=1}^{\ONum}$ \fix{and total number of episodes as $\ENum$}. 
\fix{We denote $\doteq$ as definition}.
For a given function $f:\stateSpace\rightarrow \mathbb{R}$, we define the generalized Q-function with respect to $f$ as:
\begin{equation*}
    \QFunc^{f}\paren{\state,\action} \doteq \reward\paren{\state,\action}+\expctover{\state'\sim \transDynamics \vert \state,\action}{\Func\paren{\state'}}.
\end{equation*}
When $f(\state)$ is the value function of some policy $\pi$, then the above generalized form recovers the Q-function of the policy $\pi$, i.e., $Q^\pi(\state,\action)$.
Let $\stateDist_{\policy}^t \in \Delta(\stateSpace)$ represent the distribution over states at time $t$ under policy $\pi$ given an initial state distribution $\stateDist_0 \in \Delta(\stateSpace)$. The state visitation distribution under $\pi$ can then be written as $\stateDist^{\policy}\fix{\doteq} \frac{1}{H}\sum_{t=0}^{H-1} \stateDist_t^{\policy}$. 
The value function of the policy $\pi$ under $d_0$ is then
\begin{align*}
    \begin{split}
        \VFunc^{\policy}\paren{\stateDist_0} &\fix{\doteq} \expctover{\state_0\sim \stateDist_0}{\VFunc^{\policy}\paren{\state}}\\
        &\fix{\doteq} \expctover{\state_0 \sim \stateDist_0}{\expctover{\trajectory_0\sim \rho^{\policy} \vert \state_0}{\sum_{t=0}^{H-1}\reward\paren{\state_t,\action_t}}},   
    \end{split}
\end{align*}
where $\rho^\pi(\tau_t \vert s_t)$ is the distribution over trajectories \mbox{$\tau_t = \{\state_t, \action_t, \ldots, \state_{\horizon-1}, \action_{\horizon-1}\}$}
 under policy $\pi$.
The goal is to find a policy $\policy$ that maximizes the \horizon-step
return with respect to the initial state distribution $\stateDist_0$. The associated advantage function is expressed as %
\begin{align*}
    \AFunc^{\Func}\paren{\state,\action} &\fix{\doteq}  \QFunc^{\Func}\paren{\state,\action}-\Func\paren{\state}\\
    &\fix{\doteq} \RFunc\paren{\state,\action}+\expctover{\state'\sim\transDynamics \vert \state,\action}{\Func\paren{\state'}}-\Func\paren{\state}.    
\end{align*}

\subsection{Algorithms for Learning from Multiple Oracles}
We consider a setting in which an agent has access to a set of black-box oracles $\Policies=\curlybracket{\policy^k}_{k\in\bracket{\ONum}}$ and propose several different approaches to learning from this set.

\textbf{Single-best oracle} $\policy^\star$: The most fundamental baseline is the single-best oracle $\policy^\star$, characterized by its hindsight-optimal performance, i.e., $\pi^\star \coloneqq \argmax_{\pi \in \Pi} V^{\pi}(d_0)$. This baseline is clearly not sufficient to demonstrate the superiority of the algorithm, as it does not take into account state-wise optimality of different oracles. %

\textbf{Max-following} $\policy^{\bullet}$: 
The choice of the optimal oracle varies according to the state. %
We express this \emph{state-wise expertise} by each oracle's value at the state $\VFunc^{k}(\state),\; k \in {1, \dots, \ONum}$~(we use $\VFunc^{k}(\state)$ to denote $\VFunc^{\pi^k}(\state)$ for simplicity). The \emph{max-following} baseline then selects the best-performing oracle independently at each state according to its value function, resulting in the \emph{max-following} policy $\policy^{\bullet}$
\begin{equation} \label{eq:max-follow}
    \policy^{\bullet}\paren{\action \mid \state} \doteq \policy^{k^*}\paren{\action \mid \state}, \;\; k^* \doteq \argmax_{k\in\bracket{\ONum}} \; \VFunc^{k}\paren{\state}.
\end{equation}
We can interpret \emph{max-following} policy $\policy^{\bullet}$ as a greedy policy that follows the best oracle in any state.

\textbf{Max-aggregation} $\policy^\text{max}$:  In this work, we use \emph{max-aggregation}~\citep{cheng2020policy} as our benchmark that \fix{looks one step ahead}
based on the \emph{max-following} policy $\policy^{\bullet}$. We denote a natural value baseline $\Func^\text{max}(\state_t)$ for studying imitation learning with multiple oracles as
\begin{equation}\label{eq:fmax}
    \Func^\text{max}\paren{\state} \doteq \max_{k\in \bracket{\ONum}}\VFunc^{k}\paren{\state}.
\end{equation}
We then define the \emph{max-aggregation} policy as
\begin{align}\label{eq:maxagg}
     \policy^\text{max}\paren{\action \mid \state} \doteq \delta_{\action=\action^*}, \;\action^* = \argmax_{\action\in\actionSpace}\; 
     \AFunc^{f^\textrm{max}}\paren{\state,\action},  
\end{align}
where $\delta$ is Dirac delta distribution.

When the oracle set $\Policies$ contains only a single oracle, denoted as $\policy^e$, one solution is to perform one-step policy improvement from $\policy^e$. We define the resulting policy, $\policy^e_{+}$, as follows:%
\begin{equation}
    \policy^e_{+}=\argmax_{\action\in\actionSpace}\; \RFunc\paren{\state,\action} + \expctover{\state'\sim \transDynamics \vert \state,\action}{\VFunc^{\policy^e}\paren{\state'}}.
\end{equation}
Since $\VFunc^{\policy^e_{+}}(\state) \geq \VFunc^{\policy^e}(\state)$ is guaranteed, $\policy^e_{+}$ is uniformly better than $\policy^e$ in all states. For a single oracle, $\policy^{\bullet}$%
~reduces to $\policy^e$, and $\policy^\text{max}$%
~reduces to $\policy^e_{+}$. Consequently, $\policy^\text{max}$ performs better than $\policy^{\bullet}$. In the multi-oracle case, $\policy^\text{max}$ and $\policy^{\bullet}$ are generally not directly comparable, except when one of the oracles is uniformly better than all others.
In this scenario, $\policy^{\bullet}$ reduces to $\policy^e$ and $\policy^\text{max}$ reduces to $\policy^e_{+}$. Therefore, $\policy^\text{max}$ still outperforms $\policy^{\bullet}$.

To perform online imitation learning from the max-aggregation policy, $\Func^\text{max}$ is crucial for learning the state-wise expertise from multiple oracles, as shown in Equation~\ref{eq:maxagg}. However, $\Func^\text{max}$%
~requires knowledge of  each oracle's value function. In the episodic interactive IL setting, the oracles are provided in a black-box manner, i.e., without access to their value functions. To address this issue, we follow previous work \citep{ross2010efficient,ross2014reinforcement,sun2017deeply} and reduce IL to an online learning problem.

Since the MDP transition and reward models are unknown, we regard $\stateDist^{\policy_n}$ as the adversary~(i.e., $\stateDist^{\policy_n}$ could be an arbitrary distribution) in online learning, where $\pi_n$ is the policy used for round $n$. Consequently, we define the $n$-th round online imitation learning loss as follows:
\begin{equation}\label{eqn:il-loss}
    \loss_n^{\textnormal{IL}}\paren{\policy} \doteq -\horizon \expctover{\state\sim\stateDist^{\policy_n}}{\AFunc^{f^\textrm{max}}\paren{\state,\policy}}.    
\end{equation}

In this work, we adapt the online loss $\loss_n(\policy;\lambda)$ of \citet{cheng2020policy} to balance the effect of reinforcement leanring~(i.e., explore the environment using the learner's policy only) and imitation learning~(i.e., imitating the $\policy^{\max}$ policy), resulting the following $n$-th round loss,:%
\begin{equation}\label{eq:loss}
    \begin{split}
        \loss_n\paren{\policy;\lambda} \doteq &\underbrace{-(1-\lambda)\horizon \expctover{\state\sim\stateDist^{{\policy_n}}}{\AFunc^{f^\textrm{max},\policy}_{\lambda}\paren{\state,\policy}}}_{\textnormal{Imitation Learning Loss~(Eqn.~\ref{eqn:il-loss})}}\\
        &\underbrace{-\lambda \expctover{\state\sim\stateDist_0}{\AFunc_{\lambda}^{f^\textrm{max},\policy}\paren{\state,\policy}}}_{\textnormal{Reinforcement Learning Loss}},
    \end{split}
\end{equation}
where $\Adv^{f^\textrm{max},\policy}_{\lambda}(\state,\action)$ is a $\lambda$-weighted advantage:
\begin{equation}\label{eq:lambda_weight_advantage}
    \Adv^{f^\textrm{max},\policy}_{\lambda}\paren{\state,\action} \doteq \paren{1-\lambda}\sum_{i=0}^{\infty}\lambda^{i}\Adv_{\paren{i}}^{f^\textrm{max}, \policy}\paren{\state,\action},
\end{equation}
which combines various $i$-step advantages:
\begin{equation}\label{eq:i_step_advantage}
    \begin{split}
        \Adv^{f^\textrm{max},\policy}_{\paren{i}}\paren{\state_t,\action_t} \doteq \;&\mathbb{E}_{\trajectory_t\sim \rho^\pi\paren{\cdot \mid \state_t}}
        {[\reward(\state_t,\action_t)} + \ldots \\
        &+\reward\paren{\state_{t+i},\action_{t+i}}+\Func^\text{max}\paren{\state_{t+i+1}})]\\
        &-\Func^\text{max}\paren{\state_t}. \notag
    \end{split}
\end{equation}
Consider the scenario in which we repeatedly roll-out the $k$-th oracle starting at state $s_t$.
With $N_k(s_t)$ trajectories $\trajectory_{1,k},\trajectory_{2,k},\dots,\trajectory_{\ENum_k,k}$, we compute the return estimate for state $s_t$ as the average return achieved for the trajectories,
\begin{equation}\label{eq:mu_v}
\!\!\!\!\hat{\VFunc}^{k}\paren{\state_t}\fix{\doteq} \hat{\VFunc}^{\policy_k}\paren{\state_t}\fix{\doteq} \frac{1}{\ENum_k\paren{\state_t}}\sum_{i=1}^{\ENum_k\paren{\state_t}}\sum_{j}^{\horizon}\lambda^j\RewardFunc\paren{\state_j,\action_j}, 
\end{equation}
where $\ENum_{k}(\state_t)$ is the number of trajectories starting from the initial state $\state_t$ collected by the $k$-th oracle.

\subsection{Estimator for the Policy Gradient}

We define the empirical estimate of the $\loss_n\paren{\policy,\lambda}$ gradient as%
\begin{align}\label{eq:gradient}
    &\nabla\hat{\loss}_n\paren{\policy_n;\lambda}= \\
    &-\horizon %
    \expctover{\state\sim\stateDist^{\policy_n},\action\sim\policy_n(\cdot \vert \state)}{\nabla\log\policy_n\paren{\action|\state}\AFunc^{\hat{f}^\textrm{max},\policy_n}_{\lambda}\paren{\state,\action}} \notag, 
\end{align}
where the partial derivative is taken with respect to the parameters of the policy, denoted as $\pi_n$. Since the true value function of each oracle is unknown, we use a separate function approximator $\hat{V}^k(\cdot)$ to represent the value function of each oracle $k$. The \fix{approximation error} then affects the estimation of $f^{\max}(\cdot)$, which is essential for computing the policy gradient in Equation~\ref{eq:gradient}. Therefore, the learning speed, which refers to how quickly the error in estimating $f^{\max}(\cdot)$ decreases, plays a crucial role in determining the sample efficiency. In the following sections, we discuss the limitations of the existing state-of-the-art.

\textbf{Limitations of the prior state-of-the-art:} 
A limitation of MAMBA~\citep{cheng2020policy} is its high sample complexity. MAMBA  estimates the policy gradient based on $\hat{\Func}^{\textrm{max}}(\state)$ and requires prolonged episodes to identify the optimal oracle for a given state due to its strategy of sampling an oracle uniformly at random. As a result, MAMBA suffers from a large accumulation of the error~(and hence the regret) when identification fails~(See Theorem~\ref{lem:mamba_performance_lowerbound}). Additionally, MAMBA has no control over the \fix{approximation error} of the gradient estimates when selecting which state to roll-out. Our work aims to reduce the \fix{approximation error} of the estimator by actively selecting an oracle and controlling the state-wise uncertainty through active state exploration.

%% file: algorithm.tex
\section{Algorithm}

\begin{figure}
    \centering
    \begin{subfigure}{0.45\textwidth}
    \centering
        \includegraphics[width=1.0\textwidth]{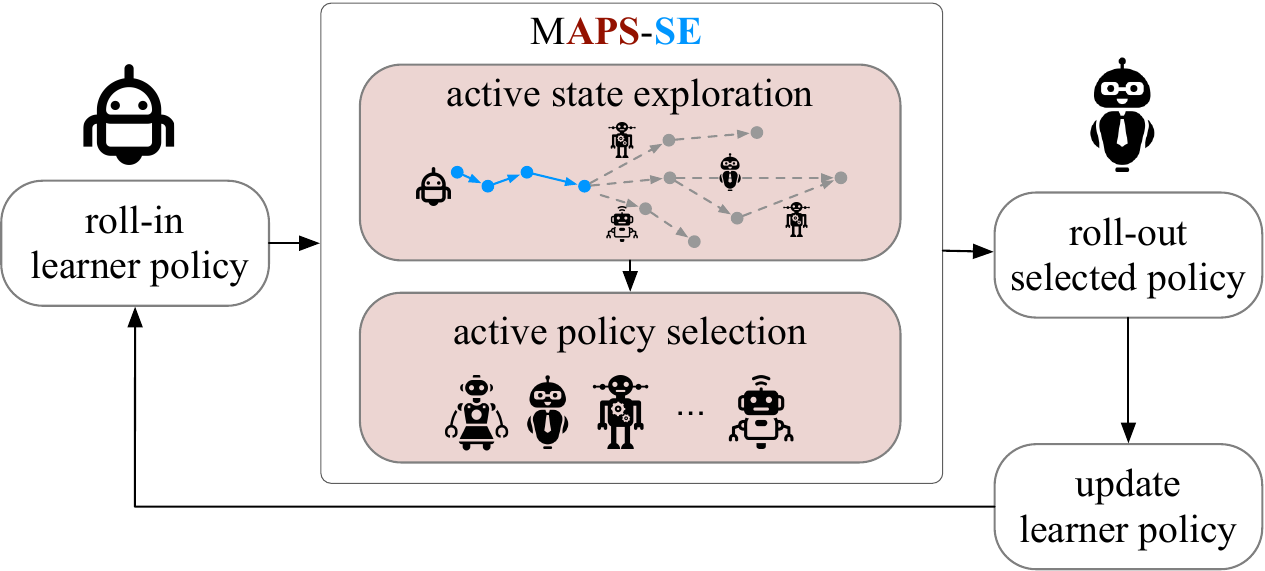}
    \end{subfigure}\hfil
    \caption{Method Overview}\label{method_overview}
\end{figure}

In this section, we focus on addressing the inefficiency of MAMBA with regards to two aspects. Firstly, MAMBA's uniform policy selection fails to effectively balance exploration and exploitation. Secondly, the selection of the state for rolling out the oracle policy plays a critical role in minimizing the sample complexity required to identify the {state-wise} optimal oracle. To address these issues, 
{we introduce \lopsase, illustrated in \figref{method_overview}, a novel policy improvement algorithm that actively selects among multiple oracles to imitate, as well as actively determines which state to explore. 
We separately introduce the active policy selection and active state exploration components in Sections~\ref{sec:alg:aps} and \ref{sec:alg:ase}, respectively. The full algorithm is outlined in \algoref{alg:lops}, with further implementation details provided in \appref{app:algdetails}. 
}

\subsection{Active Policy Selection}\label{sec:alg:aps}

\lopsase reinterprets online imitation learning as an online optimization problem, thereby permitting the utilization of any readily available online learning algorithm. For our problem, the performance~(or regret) crucially relies on the accuracy of the estimated $\hat{f}^{\max}(\cdot)$, which is computed as the maximum over the predicted value functions~($\hat{V}^k(\cdot)$) of every oracle $k \in [K]$. Consequently, this challenge necessitates the formulation of a method that identifies the optimal oracle in a sample-efficient fashion.

{For clarity and simplicity of discussion, we refer to special case of \lopsase with only the \textcolor{Maroon}{\textit{active policy selection}} component as \textcolor{Maroon}{\algname}.} This is equivalent to setting active state exploration \textcolor{NavyBlue}{$\textsc{SE} = \textsc{False}$} in Line~\ref{line:seflag} of Algorithm~\ref{alg:lops}.
\lopsaps incorporates the concept of the upper confidence bound~(UCB) to determine which oracle should be selected during the online learning process. In every state, \lopsaps decides to roll out the oracle with the highest upper confidence bound value, as this oracle has a higher probability of yielding the best performance. The data collected are subsequently used to further improve the approximation of the corresponding value function. Unlike the uniform selection strategy in MAMBA, \lopsaps establishes a more effective balance between exploration and exploitation when rolling out the oracle, thereby achieving better sample efficiency. It should be highlighted that in a single-oracle setting, $\lopsaps$
 {reduces}
to MAMBA. In the case of multiple oracles, our experimental results indicate that by applying reasoning to determine which oracle should be rolled out, \lopsaps consistently surpasses MAMBA in performance. Next, we present the details of \lopsaps for both discrete and continuous state spaces.

\begin{algorithm}[!t]
    \caption{\textbf{M}ax-aggregation \textcolor{Maroon}{\textbf{A}ctive \textbf{P}olicy Selection} with \textcolor{NavyBlue}{Active \textbf{S}tate \textbf{E}xploration} %
    (M\textcolor{Maroon}{APS}-\textcolor{NavyBlue}{SE})
    }\label{alg:lops}
    \begin{algorithmic}[1] 
        \Require {Initial learner policy $\policy_{1}$}, oracle policies $\{\policy^{k}\}_{k\in \bracket{\ONum}}$, initial value functions $\{\hat{V}^k\}_{k\in\bracket{\ONum}}$
        \For{$n=1,2, \ldots, N-1$} 
        \If{\textcolor{NavyBlue}{SE is \textsc{True}}}\label{line:seflag}
            \IndentLineComment{/* active state exploration */}
            \State \textcolor{NavyBlue}{Roll-in policy $\policy_n$ until $\Gamma_{k_\star}\paren{\state_t}$ ${\geq} \threshold$, where $k_{\star}$ and $\Gamma_{k_\star}\paren{\state_{t}}$ are computed via Equation~\ref{eq:kstar}  and \ref{eq:bonus} at each visited state $s_t$.}\label{lin:ase}
        \Else
            \State Roll-in policy $\policy_n$ up to $t_e\sim \text{Uniform}\bracket{\horizon-1}$
        \EndIf
        \IndentLineComment{/* active policy selection */}
        \State \textcolor{Maroon}{{Select $k_{\star}$ via Equation~\ref{eq:kstar}}}.\label{lin:aps}
        \State {Switch to $\policy^{k_{\star}}$ to roll-out and collect data $\mathcal{D}_n$}.
        \State Update the estimate of $\hat{V}^{k_{\star}}(\cdot)$ with $\mathcal{D}_n$.
        \State Roll-in $\policy_n$ for full $\horizon$-horizon to collect data $\mathcal{D}_n'$.
        \State Compute gradient estimator $g_n$ of $\nabla \hat{\ell}_n(\policy_n, \lambda)$ \eqref{eq:gradient} using $\mathcal{D}_n'$. \label{lin:aps:gradestimator}
        \State Update $\policy_n$ to $\policy_{n+1}$ by giving $g_n$  to a first-order online learning algorithm. 
        \EndFor 
    \end{algorithmic}
\end{algorithm}

\textbf{Discrete state space}: When the state space is discrete, we define the best oracle $k_\star$ to select for a given state $\state_{t}$ as
\begin{equation}\label{eq:alg:aps:discrete}
    k_{\star} = \argmax_{k \in \bracket{\ONum}} \hat{V}^k(s_{t}) +\sqrt{\frac{2\horizon^2\log{\frac{2}{\delta}}}{\ENum_k\paren{\state_{t}}}}, 
\end{equation}
where $\hat{V}^k(s_{t})$, $\ENum_k\paren{\state_t}$ is defined immediately following Equation~\ref{eq:mu_v}, and $\delta$ is a small-valued hyperparameter commonly seen in high-probability bounds. The exploration bonus term  $\sqrt{\frac{2\horizon^2\log{\frac{2}{\delta}}}{\ENum_k\paren{\state_{t}}}}$ is derived from 
\lemref{lem:explore_bonus}, which captures the standard deviation of the estimated value as well as our confidence over the estimation.

\textbf{Continuous state space}: In the case of a continuous state space, we employ an ensemble of prediction models to approximate the mean value $\hat{V}^{k}(s_{t})$ and the bonus representing uncertainty, denoted as $\sigma_k\paren{\state_t}$. For each oracle policy $\policy^k\in \Policies$, we initiate a set of $n$ independent value prediction networks with random values, and proceed to train them using random samples obtained from the oracle's trajectory buffer.
We formulate the UCB term and estimate the optimal oracle policy $\policy^{k_{\star}}$ using the following expression:
\begin{equation}\label{eq:alg:aps:continuous}
    k_{\star} {=} \argmax_{k \in \bracket{\ONum}} \hat{V}^k(s_{t}) + \sigma_k\paren{\state_t}. 
\end{equation}
To summarize, we determine the best oracle as
\begin{equation}\label{eq:kstar}
   \!\!\!\! {k_\star}=\argmax_{k \in \bracket{\ONum}}
    \begin{cases}
         \hat{V}^k(s_{t}) + \sqrt{\frac{2\horizon^2\log{\frac{2}{\delta}}}{\ENum_k\paren{\state_{t}}}} & \textnormal{discrete} \\
         \hat{V}^k(s_{t}) + \sigma_k\paren{\state_t} &\textnormal{continuous}
    \end{cases}
\end{equation}

\subsection{Active State Exploration}\label{sec:alg:ase}

The second limitation of MAMBA is that \fix{it doesn't reason over which state the exploration should occur.}
\fix{As a result, MAMBA may choose to roll out an oracle policy in states for which it already has good confidence on.}
Therefore, building upon \lopsaps, we propose an \emph{active state exploration} variant of \algname (\lopsase) that decides whether to continue rolling in the current learner policy or switch to the most promising oracle, similar to \lopsaps, based on an uncertainty measure for the current state. In this way, \lopsase aims to actively select the state in which to minimize uncertainty.

\fix{The bias and variance of the gradient estimates decrease when $\Func^\text{max}\paren{\state_t}$ returns the best-performing
oracle}  and the associated uncertainty of the value estimation on state $\state_t$ is minimized. For a specific state $\state_t$, \lopsase determines whether to proceed with the roll-out using the selected oracle policy (Eqn.~\ref{eq:kstar}) or continue using the learner's policy, based on the optimal oracle's uncertainty. 
The means by which we estimate the oracle's uncertainty varies depending on whether the state space is discrete or continuous.%

When the state space is discrete, \lopsase identifies the best oracle $k_{\star}$ for state $\state_t$ according to Equation~\ref{eq:alg:aps:discrete}. 
In continuous state space domains, $\ENum_{k_\star}\paren{\state_t}$ becomes intractable.
In this case, as in Section~\ref{sec:alg:aps}, we use an ensemble of value networks to measure the uncertainty $\Gamma_{k_{\star}}\paren{\state_t}$.
\lopsase then measures the exploration bonus associated with this oracle for state $\state_t$ as $\Gamma_{k_{\star}}\paren{\state_t}$
\begin{equation}\label{eq:bonus}
    \begin{split}
        \Gamma_{k_\star}\paren{\state_t}&=
        \begin{cases}
        \sqrt{\frac{2\horizon^2\log{\frac{2}{\delta}}}{\ENum_{k_\star}\paren{\state_{t}}}} &  \textrm{discrete} \\
        \sigma_{k_\star}\paren{\state_t} &  \textrm{continuous}
        \end{cases}   
    \end{split}
\end{equation}
\lopsase decides whether to roll out the best oracle in state $\state_t$ according to how confident we are in the selection of the best oracle. We define the uncertainty threshold according to \fix{\thmref{thm:ase:threshold}} as
\begin{align}
    \threshold 
    \fix{=} {\alpha\cdot\paren{\sqrt{\frac{{2\horizon^2\log\frac{2}{\delta}}}{{K + \left(\sum_i \frac{1}{\Delta_i^2}\right) \log\left(\frac{K}{\delta}\right)}}}}},
\end{align}
where $\alpha$ is a tunable hyperparameter.  %
If $\Gamma_{k_{\star}}\paren{\state_t} \geq \threshold$, \lopsase rolls out the identified oracle $k_{\star}$ at state $\state_t$. Otherwise, \lopsase transitions to the next state using the learner's policy.
Thus, by applying \lopsase, 
we aim to reduce the $f^\text{max}$ uncertainty of state $\state_t$ under learner's trajectory below the threshold $\threshold$.

%% file: analysis.tex
\section{Theoretical Analysis}

\begin{figure*}[!t]
    \centering
    \begin{subfigure}{.24\textwidth}
    \centering
        \includegraphics[width=1.0\textwidth]{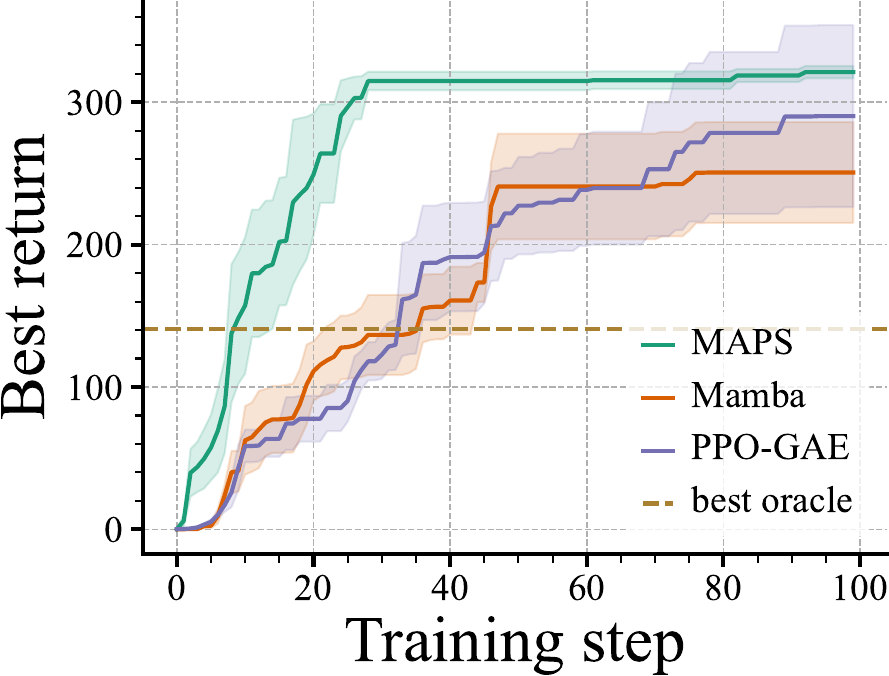}
        \caption{Cheetah-run}
    \end{subfigure}
    \begin{subfigure}{.24\textwidth}
        \centering
        \includegraphics[width=\textwidth]{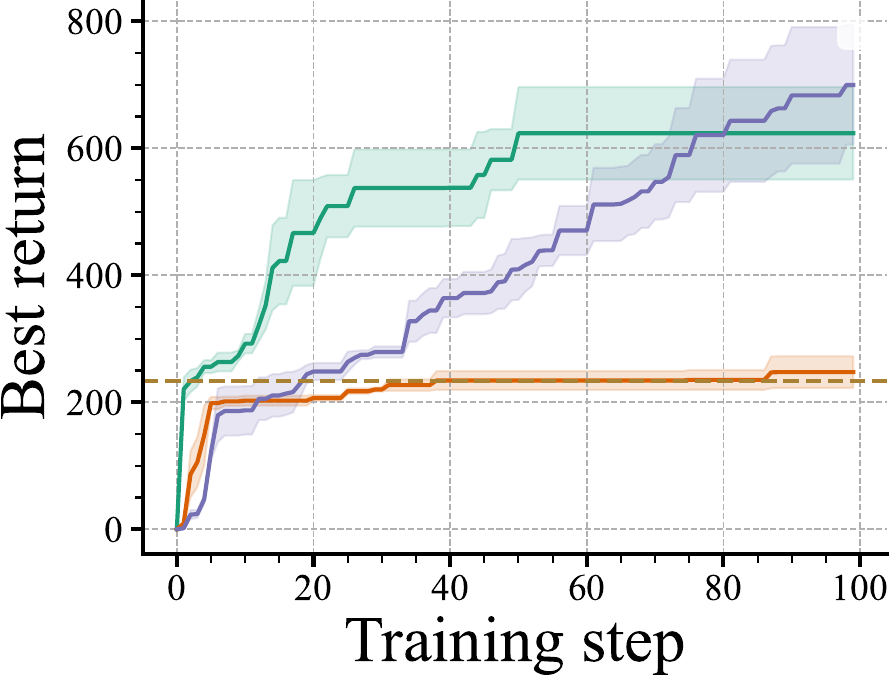}
        \caption{Cartpole-swingup}
    \end{subfigure}\hfil
    \begin{subfigure}{.24\textwidth}
        \centering
        \includegraphics[width=\textwidth]{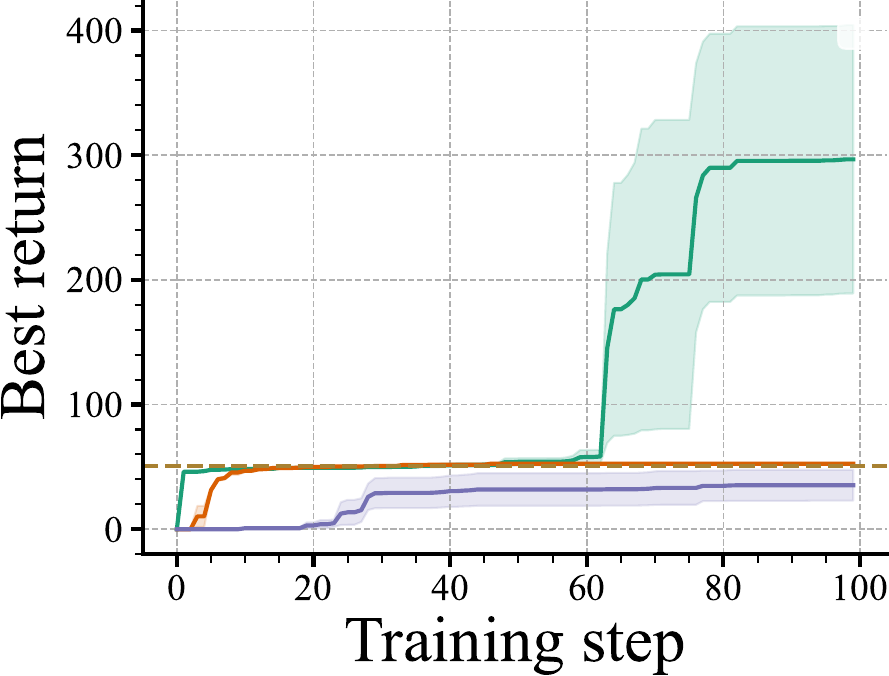}
        \caption{Pendulum-swingup}
    \end{subfigure}\hfil
    \begin{subfigure}{.24\textwidth}
        \centering
        \includegraphics[width=\textwidth]{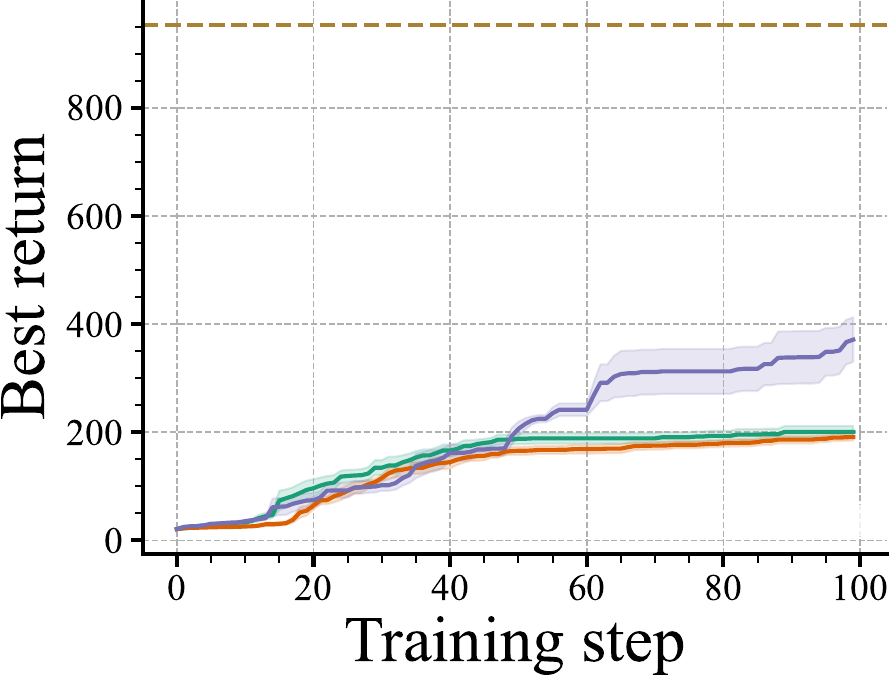}
        \caption{Walker-walk}
    \end{subfigure}\hfil

    \vspace{-0.1cm}
    \caption{Comparing the performance of \algname against the three baselines (MAMBA, PPO-GAE, and the best oracle) across four environments (Cheetah-run, Cartpole-swingup, Pendulum-swingup, and Walker-walk), using the best-return-so-far metric. Each domain includes three oracles, each representing a mixture of policies pretrained via PPO-GAE and SAC.  
    The best oracle is depicted as dotted horizontal lines in the figure. The shaded areas denote the standard error calculated using five random seeds. Except for the Walker-walk environment, \algname surpasses all baselines in every benchmark.}
    \vspace{-0.5cm}
    \label{fig:aps}
\end{figure*}
\subsection{Performance Guarantee of \algname-SE}\label{sec:thoery:ase}

In this section, we present a theoretical analysis of the benefits offered by \algname in both the APS and ASE configurations. For the sake of simplicity, our analysis concentrates on employing the online imitation learning loss $\ell^{\textnormal{IL}}_n(\cdot)$, or its equivalent $\ell_n(\cdot)$ with $\lambda=0$. This setting is the same as that in \citet{cheng2020policy}. {The proofs of the theorems in this section are deferred to the \appref{app:theory}}.

Let $\policy_n$ be the learner’s policy generated in the $n$-th round of online learning in \fix{\algoref{alg:lops}}.
Define $\zeta_N = -\frac{1}{\ENum}\sum_{n=1}^{\ENum}\loss_n^{\textnormal{IL}}\paren{\policy^{\textrm{max}}}$ and let %
\begin{align*}
    \epsilon_{\ENum}\paren{\Policies} &= \frac{1}{\ENum}\paren{\min_{\policy\in \Policies} 
 \sum_{n=1}^{\ENum}\loss_{n}^{\textnormal{IL}}\paren{\policy}-\sum_{n=1}^{\ENum}\loss_n^{\textnormal{IL}}\paren{\policy^\text{max}}},\\
    \text{Regret}_N &= \sum_{n=1}^{\ENum}\loss_{n}^{\textnormal{IL}}\paren{\policy_n}-\min_{\policy\in \Policies} \sum_{n=1}^{\ENum}\loss_n^{\textnormal{IL}}\paren{\policy}.
\end{align*}
Here, $\epsilon_{\ENum}\paren{\Policies}$ captures
the quality of $\Policies$, and $\text{Regret}_N$ characterizes the convergence rate of the online algorithm. 

\citet{cheng2020policy} establish a meta theorem for a class of max-aggregation algorithms based on the gradient estimator provided in Equation~\ref{eq:gradient}. As \algname falls into this category, the following general result also applies to \fix{\algoref{alg:lops}}:

\begin{restatable}[\citet{cheng2020policy}]{thm}{mambaPerformanceLowerBound}\label{lem:mamba_performance_lowerbound} 
Define $\zeta_N$, $\epsilon_{\ENum}\paren{\Policies}$, and $\text{Regret}_N$ as above, where $\text{Regret}_N$ corresponds to the regret of a first-order online learning algorithm based on Equation~\ref{eq:gradient}.  
It holds that
\begin{align*}
\expct{\max_{n\in\bracket{\ENum}}\VFunc^{\policy_n}\paren{\stateDist_0}}&\geq \expctover{\state\sim\stateDist_0}{\max_{k\in[K]}\VFunc^k\paren{\state}} \\
    &+ \expct{\zeta_{\ENum}-\epsilon_{\ENum}\paren{\Policies}-\frac{\text{Regret}_N}{N}},
\end{align*}
where the expectation is over the randomness in feedback and the online algorithm.
\end{restatable}
\thmref{lem:mamba_performance_lowerbound} suggests that the performance of the learning algorithm (i.e., measured by %
$\expct{\max_{n\in\bracket{\ENum}}\VFunc^{\policy_n}\paren{\stateDist_0}}$) is directly impacted by the regret $\expct{\text{Regret}_N}$. Therefore, to improve the above lower bound, it suffices to design an online learning algorithm with a better regret bound.

Consider a first-order online algorithm that satisfies \begin{equation}\label{eq:regret-ub}
\mathbb{E}\bracket{\text{Regret}_N}\leq O\paren{\beta\cdot\ENum +\sqrt{\var \ENum}},
\end{equation}
where $\beta$ and $\var$ are the bias and the variance of the gradient estimates, respectively. In the following, we show that \lopsaps and \lopsase can effectively reduce the bias term and variance term
with a smaller number
of oracle calls per state compared to its passive counterpart.

\subsection{Advantage of \lopsaps Over Uniform Policy Sampling}

We consider the discrete-state setting and provide an upper bound on the number of state visitations required by \textrm{\algname}
to identify the best oracle at any given state $s$.

\begin{restatable}[]{thm}{SampleComplexityAPS}\label{thm:APS-sample-complexity}
For any state $s$ and $\ONum$ oracles, the best oracle an be identified with probability at least $1-\delta$ using {\algoref{alg:lops}} %
with active policy selection\footnote{For analysis, we assume that the upper confidence bound of each oracle $i$ is computed by $\widehat{V}_{i}(s)+\sqrt{a/N_i(t)}$ for each iteration $t$, and $0\leq a \leq (25(T-K))/(36(\sum_i \Delta_i^{-2}))$ for simplicity.}  if the number of visitation on this state is
\begin{align}
    T = \mathcal{O}\paren{\ONum + \paren{\sum_i \frac{\horizon^2}{\Delta_i(s)^2} }\log \paren{\frac{\ONum}{\delta}}},   
\end{align}
where $\Delta_i(s) := \max_{j\neq i^{\star}} V_{i^{\star}}(s) - V_{j}(s)$, $i^{\star}=\max_i V_i(s)$ is the suboptimality gap of oracle $i$ at state $s$, and $\horizon$ is the task horizon. 
\end{restatable}

{%
Note that the bias term $\beta$ in Equation~\eqref{eq:regret-ub} is caused by two factors: (a) the bias in estimating ${\VFunc}^{k_\star}\paren{\state}$ when $\Func^{\textrm{max}}$ actually selects the optimal oracle at state $\state$, and (b) the bias due to $\Func^{\textrm{max}}$ selecting the suboptimal oracle at state $\state$. The bias from (a) will be eliminated as $N$ grows. According to \thmref{thm:APS-sample-complexity}, the bias introduced by (b) over $N$ episodes is upper bounded by $T\beta^{\textrm{max}}/N$, where $\beta^{\textrm{max}}$ denotes the maximum difference in loss between $\pi^n$ and the optimal oracle at $s$. Suppose that the learner visited $\stateSpace'\subseteq\stateSpace$ states in total, with $|\stateSpace'|\leq \stateSpace$. Then by the union bound, we get with probability $1-|\stateSpace'|\delta$, the biased caused by (b) is 
\begin{align*}
    \mathcal{O}\paren{ \paren{\ONum + \paren{\frac{K\horizon^2}{\min_{i,s}\Delta_{i}(s)^2} }\log \paren{\frac{\ONum}{\delta}} } \frac{|\stateSpace'|\beta^{\textrm{max}}}{N} }.
\end{align*}
}

In contrast, if we replace \linref{lin:aps} of \fix{\algoref{alg:lops}} by uniform sampling, \lopsaps reduces to \mamba, and we have
\begin{restatable}[]{thm}{SampleComplexityMamba}\label{thm:Mamba-sample-complexity}
Under the same conditions as in Theorem~\ref{thm:APS-sample-complexity}, if we adopt the uniform selection strategy, then with probability at least $1-\delta$, the best oracle can be identified if
\begin{align}
    T = \mathcal{O}\paren{ \paren{\sum_i \frac{\ONum \horizon^2}{\Delta_i(s)^2}} \log \paren{\frac{\ONum}{\delta}}}.
\end{align}
\end{restatable}
The uniform policy selection strategy incurs a cost of $\widetilde{\mathcal{O}}(KH^2\log(K))$, while \algname
necessitates merely $\widetilde{\mathcal{O}}(K+H^2\log(K))$ oracle calls. Consequently, when dealing with a large number of oracles (that is, when $K$ is sizeable), the superiority of \algname
over the uniform policy selection strategy becomes increasingly pronounced.

\algname-SE,
with its active state exploration, bypasses unnecessary exploration in states for which we are sufficiently confident regarding the best oracle based on its state value. Theorem~\ref{thm:ase:threshold} provides a stopping criterion for \lopsase: 

{
\begin{restatable}[]{thm}{LOPSASELowerBound}\label{thm:ase:threshold}
 For any state $\state$ and $\ONum$ experts, with probability at least \fix{$1-\delta$}, \algname-SE identifies the best oracle when 
 the exploration bonus reaches the uncertainty threshold
    \[ \threshold = o\paren{ \sqrt{\frac{{2\horizon^2\log\frac{4}{\delta}}}{{K + \left(\sum_i \frac{\horizon^2}{\Delta_i^2}\right) \log\left(\frac{2K}{\delta}\right)}}}},\]
    where $\Delta_i(s) := \max_{j\neq i^{\star}} V_{i^{\star}}(s) - V_{j}(s)$, $i^{\star}=\max_i V_i(s)$ is the suboptimality gap of oracle $i$ at state $s$.
\end{restatable}
When $\threshold$ is set too large, the learner runs the risk of not calling any oracle and consequently having too much uncertainty on the estimate $\hat{V}^{k_\star}(s)$, leading to a large bias in the gradient estimates (and therefore the regret). On the other hand, when $\threshold$ is too small, the learner may stop rolling in $\pi_n$ early (at Line~\ref{lin:ase} of Algorithm~\ref{alg:lops}), and therefore waste collecting samples on states that are sufficient confident. In the next section, we will show that with proper choices of $\Gamma_s$ as suggested by \thmref{thm:ase:threshold}, \algname-SE strikes a good empirical balance between exploring uncertainty states and collecting useful training examples for improving $\pi_n$.
}

%% file: experiments.tex
\section{Experiments}

In this section, we perform an empirical study of \algname and \algname-SE, comparing them to the best-oracle, PPO-GAE, and MAMBA baselines. We find that both \algname and \algname-SE outperform these baselines in most scenarios.

\subsection{Experiment Setup}

\textbf{Environments.} We evaluate our method on four continuous state and action environments: Cheetah-run, CartPole-swingup, Pendulum-swingup, and Walker-walk, which are part of the DeepMind Control Suite \citep{tassa2018deepmind}.

\textbf{Oracle Policies.} We train the oracle policies for each environment using proximal policy optimization (PPO)~\citep{schulman2017proximal} integrated with a generalized advantage estimate (GAE)~\citep{schulman2015high} and soft actor-critic (SAC)~\citep{haarnoja2018soft}. The weights of the learner policy are periodically saved as checkpoints. Generally, the average performance of the oracles increases monotonically as training progresses, which implies that each checkpoint represents an oracle of progressively improving quality.

\textbf{Baseline Methods.} We evaluate \algname and \algname-SE against three representative baselines: (1) the best oracle; (2) proximal policy optimization with a generalized advantage estimate (PPO-GAE) serving as a pure RL baseline; and (3) MAMBA, the  state-of-the-art method for online imitation learning from multiple black-box oracles. For further details, we refer the reader to \appref{app:experiment:baselines}.

{\textbf{Setup.} To guarantee a fair evaluation, \algname, \algname-SE, MAMBA, and PPO-GAE are assessed based on an equal number of environment interactions. Specifically, for \algname, \algname-SE, and MAMBA, the oracle roll-out is used to update the learner policy, with each training iteration involving RIRO through both the learner's policy and the selected oracle's policy. Hence, we ensure that the total number of environment interactions for PPO-GAE 
is the same as that of \algname, \algname-SE and MAMBA.

\begin{figure}[t]
    \centering
    \rotatebox{90}{ \quad \quad \quad \scriptsize {Freq. of selected oracle }}
    \begin{subfigure}{0.475\linewidth}
        \includegraphics[width=1.0\linewidth,  clip={0,0,0,0}]{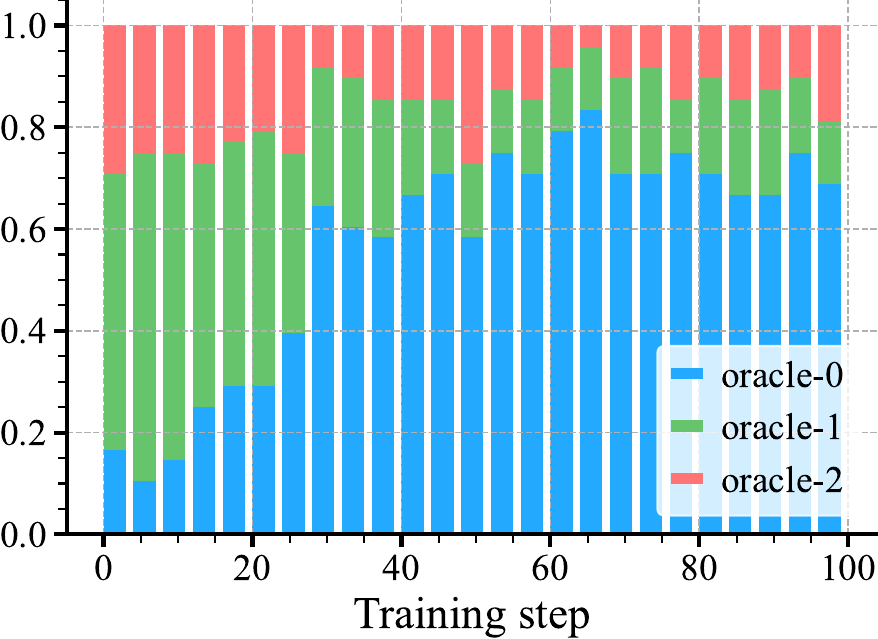}
        \caption{\algname}
    \end{subfigure}\hfil
    \begin{subfigure}{0.475\linewidth}
        \includegraphics[width=1.0\linewidth,  clip={0,0,0,0}]{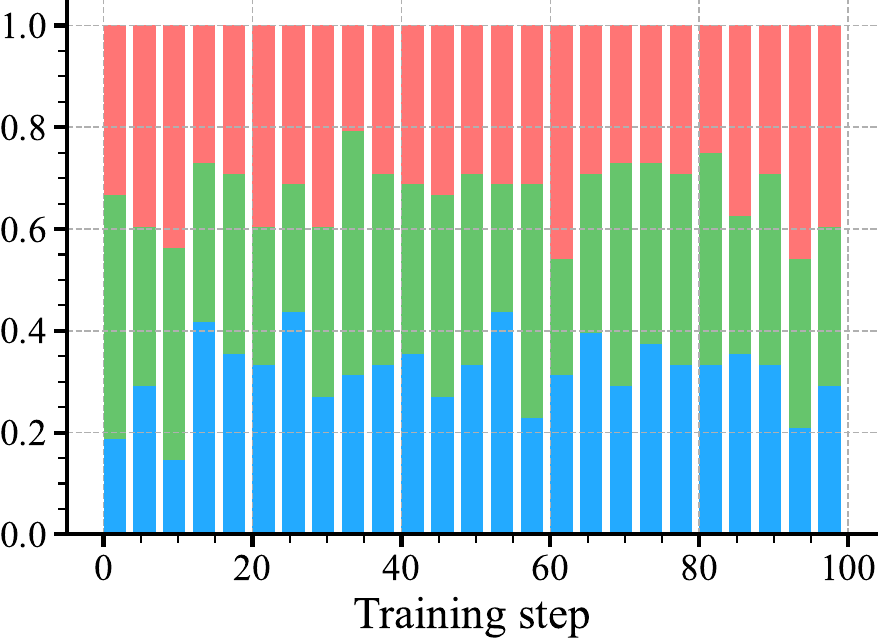}
        \caption{MAMBA}
    \end{subfigure}
    \caption{A comparison of the frequency with which \textbf{(a)} \algname and \textbf{(b)} MAMBA select among a bad (in red), mediocre (in green), and good (in blue) oracle in a three-oracle Cheetah-run experiment. \algname efficiently identifies each oracle's quality as indicated by the frequency with which
    it queries the good oracle. In contrast, MAMBA maintains roughly the same selection frequency for the bad, mediocre, and good oracles throughout.}
  \label{fig:aps:select_oracle}
  \vspace{-0.5cm}
\end{figure}

\subsection{Active Policy Selection }

\paragraph{Performance}  Figure~\ref{fig:aps} compares \algname against the best oracle, PPO-GAE, and MAMBA on Cheetah-run, Cartpole-swingup, Pendulum-swingup, and Walker-walk with a multi-oracle set. We observe that \algname 
outperforms MAMBA and the other baselines including the best oracle in all domains except for Walker-walk. In the case of Walker-walk, we suspect that relatively high quality of the oracles \fix{with a limited transition buffer size} result in less accurate value function estimates in states that the learner encounters early in training.
Further, we see that the performance \algname improves sooner during training, demonstrating the sample efficiency advantages of \algname.

}

\paragraph{Effect of active policy selection}

The superior performance of \algname can be explained by how \algname actively selects a better oracle, in contrast with the random selection process employed by MAMBA. Figure~\ref{fig:aps:select_oracle}  illustrates the frequency with which each oracle is chosen by an algorithm at each training iteration. \algname indeed has a clear preference in its oracle selection,
quickly identifying the qualities of the oracles and efficiently learning from the best one to improve performance.
This observation is aligned with the results from~{\thmref{thm:APS-sample-complexity}} %
and is directly evident from the improved performance and decreased \fix{bias term}
of the gradient estimates. In contrast, MAMBA {spares} a lot of resources to improve the value function estimate of bad or mediocre oracles that do not necessarily help the learner policy improve, negatively affecting sample efficiency.

\subsection{Active State Exploration}

\begin{figure}[t]
\centering
    \begin{subfigure}{0.235\textwidth}
        \centering
        \includegraphics[width=\linewidth]{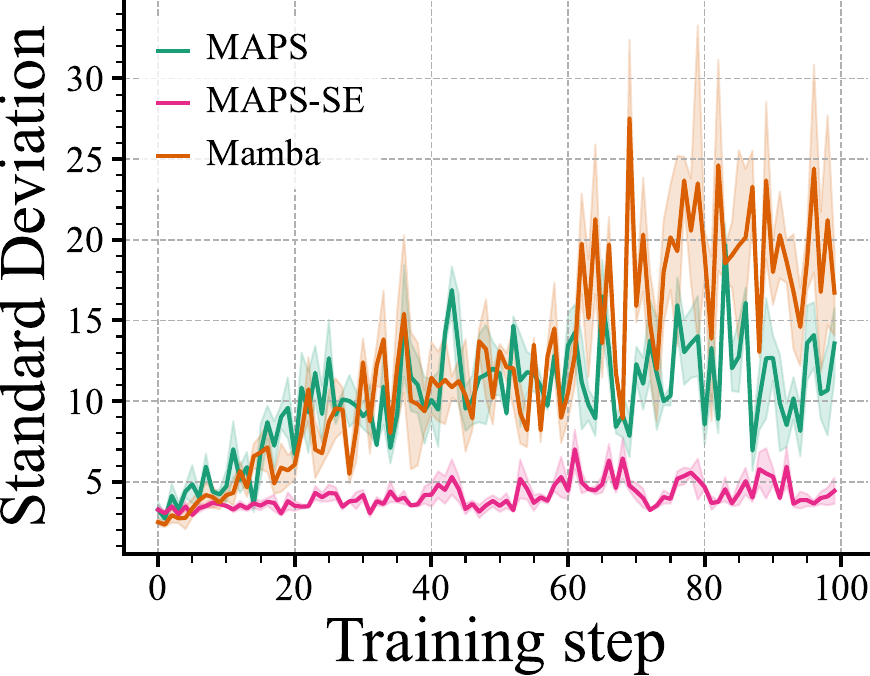}
    \end{subfigure}
    \begin{subfigure}{0.235\textwidth}
        \centering
            \includegraphics[width=\linewidth]{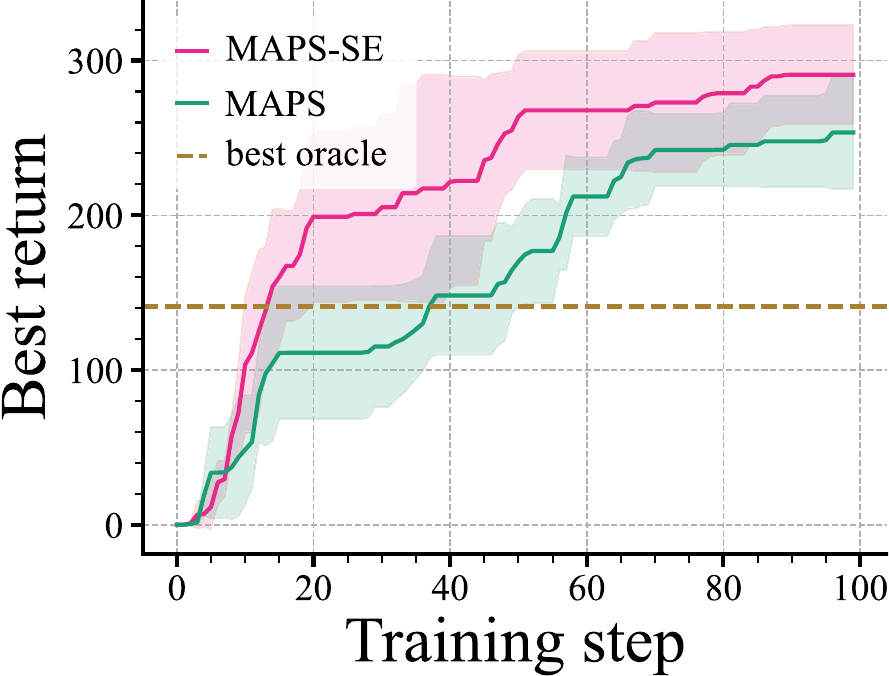}
    \end{subfigure}
    \caption{
    \textbf{Left}: A demonstration of the benefits of active state exploration in terms of a comparison between the standard deviation for switch-state for \algname-SE (in blue), {\algname (in orange)} and MAMBA (in green) in the Cheetah-run environment with the same set of oracles used in Figure~\ref{fig:aps}.
    under a threshold $\threshold=2.5$. The predicted standard deviation is evaluated at the switching state from learner policy to oracle. \textbf{Right}: A comparison of the best-return-so-far \fix{in a}
    multiple oracle set between \fix{\algname-SE with \algname and the best oracle baseline} on the Cheetah-run environment.}\label{fig:ase:effect}
    \vspace{-0.5cm}
\end{figure}

\paragraph{Performance} Active state exploration requires a threshold, $\threshold$ that indirectly modulates the uncertainty of a state's predicted value. The theoretical analysis in \thmref{thm:ase:threshold} reveals that this threshold is directly connected to the \fix{bias term}
of the gradient estimates. In essence, the threshold $\threshold$ trades off between 
\fix{bias term}
and sample efficiency. When the threshold is set sufficiently high, the learner policy never transitions to an oracle, resulting in substantial \fix{bias} 
error due to the inaccurate estimate of the oracles' value functions. On the other hand, an extremely low threshold prompts the learner to immediately switch to an oracle, unless the uncertainty of the predicted value diminishes significantly. Consequently, the model spends large number of environment steps to improve the estimation of the value function.

As illustrated in Figure~\ref{fig:ase:effect}, the \algname-SE algorithm outperforms the \algname algorithm by preventing unnecessary exploration. 
{However, it is worth noting that \algname-SE has limitations regarding its dependence on the environment and the need for non-trivial hyper-parameter tuning.}

\paragraph{Effect of active state exploration} Figure~\ref{fig:ase:effect} illustrates that 
active state exploration significantly decreases the standard deviation of the predicted value at switching states. 
This is in contrast to the increasing trend of MAMBA {and \algname}.
MAMBA randomly chooses a switching timestep as well as the oracle to switch to. This risks exposing the value function (of the selected oracle) to an unfamiliar state, resulting in an increase in the standard deviation. 
A small standard deviation in \algname-SE \fix{may result from a small variance $\var$} of the gradient estimates, resulting in improved performance.

%% file: conclusion.tex
\section{Conclusion}
We've introduced \algname, a novel active policy improvement algorithm that learns from multiple black-box oracles.
It is motivated by real-world scenarios where one wants to learn a state-dependent policy efficiently from multiple suboptimal oracles. \algname incorporates active policy selection and active state exploration that provably improve sample complexity over the current state-of-the-art approach, MAMBA. Empirically, our results demonstrate that \algname excels at identifying the state-wise quality of black-box oracles, learns from the best oracle 
\fix{more efficiently than the other baselines on various benchmarks.}

\section*{Acknowledgements}
We thank Ching-An Cheng for constructive suggestions. We also thank Yicheng Luo, Ziyu Ye and Zixin Ding for initial discussion. This work is supported in part by the RadBio-AI project (DE-AC02-06CH11357), U.S. Department of Energy Office of Science, Office of Biological and Environment Research, the Improve project under contract (75N91019F00134, 75N91019D00024, 89233218CNA000001, DE-AC02-06-CH11357, DE-AC52-07NA27344, DE-AC05-00OR22725), the Exascale Computing Project (17-SC-20-SC), a collaborative effort of the U.S.\ Department of Energy Office of Science and the National Nuclear Security Administration (DOE DE-EE0009505), and NSF HDR TRIPODS (2216899).

%% file: supp_notation.tex
\section{Notations}\label{app:notations}

Table~\ref{table:notation} summarizes the notation used in the main paper.
\begin{table}[ht]
\centering
\vspace{-0.2cm}
\scalebox{0.85}{
\begin{tabular}{l l}
\toprule
\textbf{Notation} & \textbf{Description} \\
\midrule
&\textbf{\quad\quad\quad\quad\quad Problem Statement}\\
$\VFunc$, $\widehat{\VFunc}$ & value function , value function estimator \\
$\MDP_0$&$\paren{\stateSpace,\actionSpace,\transDynamics,\RFunc,\discFactor}$, finite-horizon MDP\\
$\stateSpace$ & state space\\
$\actionSpace$ & action space\\
$\transDynamics$, $\transDynamics\paren{\state'|\state,\action}$ & transition dynamics, $\transDynamics: \stateSpace \times \actionSpace \rightarrow \Delta \stateSpace$\\
$\ENum$ & total number of episodes\\
$\RFunc$, $\RFunc\paren{\state,\action}$ & reward function, $\RFunc:\stateSpace\times\actionSpace\rightarrow\bracket{0,1}$ \\
$\Policies$ & $\Policies=\curlybracket{\policy^k},k\in\bracket{\ONum}$\\
$\policy$ & policy, $\policy: \stateSpace \rightarrow \Delta \actionSpace$ maps the current state to a distribution over actions\\
{$\policy^k$, $\policy_n$ } & {$k_{th}$ oracle policy, $n_{th}$ learner policy} \\
$\Func$&$f:\stateSpace\rightarrow \mathcal{R}$\\
$\QFunc_t^{\policy}\paren{\state,\action}$&$ \reward\paren{\state,\action}+\expctover{\state'\sim \transDynamics|\state,\action}{\Func\paren{\state'}}$\\
$\stateDist_{\policy}^t$ & the distribution over states at time $t$\\
$\dataset_{\policy}$ & dataset\\
$\discFactor$ & discount factor\\
$\lambda$ & blending value of RL and IL objective\\
$\trajectory_t$& $ \state_t, \action_t,..., \state_{\horizon-1},\action_{\horizon-1}$ \\
$\rho^{\policy}\paren{\trajectory_t|\state_t}$ & trajectory distribution\\
$\AFunc^{\Func}\paren{\state,\action}$  & advantage function\\
$\policy^{\bullet}\paren{\action|\state}$ & $\policy^{k_s}\paren{\action|\state}, \textrm{where } k_s:=\argmax_{k\in\bracket{\ONum}} \VFunc^{k}\paren{\state} $\\
$\Func^{max}\paren{\state_t}$&$\max_{k\in \bracket{\ONum}}\VFunc^{k}\paren{\state}$\\
$\policy^{max}\paren{\action|\state}$ & $\delta_{\action=\action_s}, \textrm{where } \action_s:=\argmax_{\action\in\actionSpace} \AFunc^{max}\paren{\state,\action}$\\
$\policy^e$ & single oracle\\
$\policy^e_{+}$ & one-step policy improvement from $\policy^e$\\
$\Adv^{max,\policy}_{\lambda}\paren{\state,\action}$ &  $\lambda$-weighted advantage\\
$\loss_n\paren{\policy;\lambda}$&$-(1-\lambda)\horizon \expctover{\state\sim\stateDist}{\AFunc^{max,\policy}_{\lambda}\paren{\state,\policy}}-\lambda \expctover{\state\sim\stateDist_0}{\AFunc_{\lambda}^{max,\policy}\paren{\state,\policy}}$\\
$\triangledown\widehat{\loss}_n\paren{\policy_n;\lambda}$&$ -\horizon \mathbb{E}_{\state\sim\stateDist^{\policy_n}}\expctover{\action\sim\policy|\state}{\triangledown\log\policy\paren{\action|\state}\widehat{\AFunc}^{\policy}_{\lambda}\paren{\state,\action}}|_{\policy=\policy_n}$\\
$\widehat{\VFunc}^{\policy_k}\paren{\state_t}$&$\frac{1}{\ENum_k\paren{\state_t}}\sum_{i=1}^{\ENum_k\paren{\state_t}}\sum_{j}^{\horizon}\lambda^j\RewardFunc\paren{\state_j,\action_j}$\\
$\ENum_{k}\paren{\state_t}$ & the number of trajectories that start from initial state $\state_t$ collected by the $k_{th}$ oracle\\
$\ENum\paren{\state_t}$ & the number of trajectories that start from initial state $\state_t$\\
\hline
&\textbf{\quad\quad\quad\quad\quad Algorithm}\\
$\textrm{APS}$ & Active Policy Selection\\
$\textrm{ASE}$ & Active State Exploration\\
$\Gamma_{k_{\star}}\paren{\state_t}$, $\sigma_k\paren{\state_t}$ & exploration bonus, uncertainty\\
$\threshold$ & threshold\\
$t_e$ & the round switch to oracle from learner \\
$\state_{t_e}$ & switching state \\
\hline
&\textbf{\quad\quad\quad\quad\quad Analysis}\\
{$\zeta_N$} &$ -\frac{1}{\ENum}\sum_{n=1}^{\ENum}\loss_n\paren{\policy^{\textrm{max}}}$\\
$\epsilon_{\ENum}\paren{\Policies}$ & $\min_{\policy\in \Policies} \frac{1}{\ENum}\paren{\sum_{n=1}^{\ENum}\loss_{n}\paren{\policy}-\sum_{n=1}^{\ENum}\loss_n\paren{\policy^\text{max}}}$\\

{$\text{Regret}_N$ }&    $ \text{Regret}_N = \sum_{n=1}^{\ENum}\loss_{n}\paren{\policy_n}-\min_{\policy\in \Policies} \sum_{n=1}^{\ENum}\loss_n\paren{\policy}$ \\

$\Delta_N$ & $\frac{-1}{\ENum}\sum_{n=1}^{\ENum}\loss_n\paren{\policy^{\textrm{max}}}$\\
$\Delta_i(s)$ & $ V_{i^\star}(s) - V_{i}(s)$\\
$\beta^{\textrm{max}}$ & 
the maximum difference in loss between $\pi^n$ and the optimal oracle at $s$\\
$\var$ & variance\\
\bottomrule
\end{tabular}
}
\caption{Notations used in the main paper}
\label{table:notation}
\vspace{-10mm}
\end{table}

\clearpage

%% file: supp_algo_characteristics.tex
\section{Related Work: Problem Setup}\label{app:problem_setup}
For better positioning of this work, we highlighted the key differences and compared our setting against a few related works in this domain in the problem setup in \tabref{table:alg_characteristics}.

\begin{table*}[ht]
\scalebox{0.9}{
\begin{tabular}{l l l l l l l l }
\toprule
\textbf{Algorithm}
& \textbf{Criterion }
& \textbf{Online}
& \textbf{Stateful}
& \textbf{Active}
& \textbf{Interactive}
& \makecell[c]{\textbf{Multiple}\\ \textbf{oracles}}
& \makecell[c]{\textbf{Sample} \\ \textbf{efficient} \\ \fix{\textbf{under multi-oracles}} }
\\
\hline
{\makecell[l]{Behavioral cloning\\{\citep{pomerleau1988alvinn}}}}
&  IL 
& \makecell[c]{$\times$ }
& \makecell[c]{$\checkmark$}
& \makecell[c]{$\times$ }
& \makecell[c]{$\times$ }
& \makecell[c]{$\times$ }
& \makecell[c]{$-$ }
\\
\hline
{\makecell[l]{{SMILE}\\\citep{daume2009search}}}
&  IL 
& \makecell[c]{$\times$ }
& \makecell[c]{$\checkmark$}
& \makecell[c]{$\times$ } 
& \makecell[c]{$\times$ }
& \makecell[c]{$\times$ }
& \makecell[c]{$-$ }
\\
\hline
{\makecell[l]{DAgger \\\citep{ross2010efficient}}}
& IL
& \makecell[c]{$\checkmark$}
& \makecell[c]{$\checkmark$}
& \makecell[c]{$\times$ } 
& \makecell[c]{$\checkmark$}
& \makecell[c]{$\times$ }
& \makecell[c]{$-$ }
\\
\hline
{\makecell[l]{AggreVaTe \\{\citep{ross2014reinforcement}} }}
& IL
& \makecell[c]{$\checkmark$}
& \makecell[c]{$\checkmark$}
& \makecell[c]{$\times$ } 
& \makecell[c]{$\checkmark$}
& \makecell[c]{$\times$ }
& \makecell[c]{$-$ }
\\
\hline
{\makecell[l]{PPO with GAE \\{\citep{schulman2017proximal, schulman2015high}} }}
& RL
& \makecell[c]{$\checkmark$}
& \makecell[c]{$\checkmark$}
& \makecell[c]{$\times$ } 
& \makecell[c]{$\times$ }
& \makecell[c]{$\times$ }
& \makecell[c]{$\times$ }
\\
\hline
{\makecell[l]{AggreVaTeD \\{\citep{sun2017deeply}} }}
& IL
& \makecell[c]{$\checkmark$}
& \makecell[c]{$\checkmark$}
& \makecell[c]{$\times$ } 
& \makecell[c]{$\checkmark$}
& \makecell[c]{$\times$ }
& \makecell[c]{$-$ }
\\
\hline
{\makecell[l]{LEAQI \\{\citep{brantley2020active}}  }}
& IL
& \makecell[c]{$\checkmark$}
& \makecell[c]{$\checkmark$}
& \makecell[c]{$\checkmark$}
& \makecell[c]{$\checkmark$}
& \makecell[c]{$\times$ }
& \makecell[c]{$-$}
\\
\hline
{\makecell[l]{MAMBA \\{\citep{cheng2020policy}}  }}
&  IL $+$ RL 
& \makecell[c]{$\checkmark$}
& \makecell[c]{$\checkmark$}
& \makecell[c]{$\times$ } 
& \makecell[c]{$\checkmark$}
& \makecell[c]{$\checkmark$}
& \makecell[c]{$\times$ }
\\
\hline
{\makecell[l]{A-OPS \\{\citep{konyushova2021active}} }}
& Policy Selection
& \makecell[c]{$\times$}
& \makecell[c]{$\times$ }
& \makecell[c]{$\checkmark$}
& \makecell[c]{$\times$ }
& \makecell[c]{$\checkmark$}
& \makecell[c]{$\checkmark$}
\\
\hline
{\makecell[l]{ILEED\\{\citep{ilbelaiev}}}}
& IL
& \makecell[c]{$\times$}
& \makecell[c]{$\checkmark$}
& \makecell[c]{$\times$ } 
& \makecell[c]{$\times$ }
& \makecell[c]{$\times$ }
& \makecell[c]{$\times$ }
\\
\hline
{\makecell[l]{CAMS\\{\citep{liu2022cost}} }}
& Model Selection
& \makecell[c]{$\checkmark$}
& \makecell[c]{$\times$ }
& \makecell[c]{$\checkmark$}
& \makecell[c]{$\times$ }
& \makecell[c]{$\checkmark$}
& \makecell[c]{$\checkmark$}
\\
\hline
{\makecell[l]{\algname\\(ours) }}
& IL $+$ RL 
& \makecell[c]{$\checkmark$}
& \makecell[c]{$\checkmark$}
& \makecell[c]{$\checkmark$}
& \makecell[c]{$\checkmark$}
& \makecell[c]{$\checkmark$}
& \makecell[c]{$\checkmark$}
\\
\bottomrule
\end{tabular}
}
\caption{Algorithms Characteristics}
\label{table:alg_characteristics}
\end{table*}

\section{\fix{Related Work: Theoretical Guarantees}}
We summarize the sample complexity {for identifying the best oracle per state} of related algorithms in \tabref{table:complexity}.

\begin{table}[ht]
\centering
\scalebox{0.85}{
\begin{tabular}{l c c}
\toprule
\textbf{Selection strategy}
& \textbf{Sample complexity} & \textbf{\threshold}
\\
\midrule
Uniform~(MAMBA)
&  $\mathcal{O}\paren{ \paren{\sum_i \frac{\ONum \horizon^2}{\Delta_i^2}} \log \paren{\frac{\ONum}{\delta}}}$ & ---
\\
APS~(\algname)
&  \makecell[c]{$\mathcal{O}\paren{\ONum + \paren{\sum_i \frac{\horizon^2}{\Delta_i^2} }\log \paren{\frac{\ONum}{\delta}}}$} & ---\\
ASE~(\algname-SE) & $\mathcal{O}\paren{\ONum + \paren{\sum_i \frac{\horizon^2}{\Delta_i^2} }\log \paren{\frac{\ONum}{\delta}}}$  &  ${o\paren{\sqrt{\frac{{2\horizon^2\log({4}/{\delta})}}{{K + \left(\sum_i {\horizon^2}/{\Delta_i^2}\right) \log\left({2K}/{\delta}\right)}}}}}$\\
\bottomrule
\end{tabular}}
\caption{APS~(\algname) achieves a significant reduction in sample complexity of scale \ONum compared to uniform~(MAMBA). ASE~(\algname-SE) exhibits the same sample complexity as APS, provided that a pre-set \threshold~condition is met.} \label{table:complexity}
\end{table}

\clearpage

%% file: supp_theory.tex
\section{Proofs for Theorems}\label{app:theory}

\SampleComplexityAPS*
\begin{proof}
We first define the gap, for $i\neq i^\star$,
\begin{align}
    \Delta_i\paren{\state} = \VFunc_{i^\star}\paren{\state} - \VFunc_{i}\paren{\state}.  
\end{align}
For the case of $i=i^\star$,
\begin{align}
    \Delta_{i^\star} = \min_{i\neq i^\star} \VFunc_{i^\star}\paren{\state} - \VFunc_{i}\paren{\state}. 
\end{align}
Then, we define the following event for a state $s$ %
\begin{align}
    E\paren{\state} = \left\{\forall i \in [\ONum], t \in [T], \left|\widehat{V}_{i, t}\paren{\state} - \VFunc_i\paren{\state}\right|\leq \frac{1}{5}\sqrt{\frac{a}{t}} \right\}
\end{align}

By Hoeffding's inequality, we will have that 
\begin{align}
    \Pr{\left|\widehat{V}_{i, t}\paren{\state} - \VFunc_i\paren{\state}\right|\leq \frac{1}{5}\sqrt{\frac{a}{t}} } \geq 1 - 2\exp \paren{\frac{-a}{50 \horizon^2}}.
\end{align}
By applying the union bound, we have
\begin{align}
    \Pr{E\paren{\state}} \geq 1 - 2T\ONum\exp \paren{\frac{-a}{50 \horizon^2}}.
\end{align}

In the next, we will show that, under the event $E\paren{\state}$, we identified the best policy if
\begin{align}\label{eqn:to-show}
    \frac{1}{5}\sqrt{\frac{a}{N_i(T)}} \leq \frac{\Delta_i}{2}, \; \forall i \in [\ONum],
\end{align}
where $N_i(t)$ denotes the number of i.i.d samples used for estimating the value function of oracle $i$ at $t_{th}$ iteration. If the above holds, then we will have for $i\neq i^\star$,
\begin{subequations}
    \begin{align}
        \widehat{\VFunc}_{i, N_i(T)}(\state) &\leq \VFunc_{i}(\state) + \frac{1}{5}\sqrt{\frac{a}{N_i(T)}} \\
        &\leq \frac{\VFunc_{i^\star}(\state) + \VFunc_{i}(\state)}{2}.\label{eqn:upperbound-v-i}
    \end{align}
\end{subequations}
For $i=i^\star$, we simply have
\begin{subequations}
    \begin{align}
        \widehat{\VFunc}_{i^\star, N_{i^\star}(T)}(\state) &\geq V_{i^\star}(\state) - \frac{1}{5} \sqrt{\frac{a}{N_{i^\star}(T)}} \\
        &\geq V_{i^\star}(\state) - \frac{\Delta_{i^\star}}{2} \\
        &= V_{i^\star}(\state) - \frac{\min_{i:i\neq i^\star} V_{i^\star}(\state) - V_{i}(\state)}{2} \\
        &=\frac{\max_{i:i\neq i^\star} V_{i}(\state) + V_{i^\star}(\state)}{2}.\label{eqn:lowerbound-v-istar}
    \end{align}
\end{subequations}
Hence, by the results from equation~\ref{eqn:lowerbound-v-istar} and equation~\ref{eqn:upperbound-v-i}, we have
\begin{align}
    \widehat{\VFunc}_{i^\star, N_{i^\star}(T)}(\state) \geq  \widehat{\VFunc}_{i, N_i(T)}(\state) .
\end{align}
 In the next, we will prove Equation~\ref{eqn:to-show} is true. This is equivalent to show the following,
\begin{align}
    N_i(T) \geq \frac{4}{25}\frac{a}{\Delta_i^2}
\end{align}

 We first show that 
\begin{align}
    N_i(t) \leq \frac{36}{25} \frac{a}{\Delta_i^2} + 1,\; \forall i\neq i^\star.
\end{align}
We prove the above result by induction. Firstly, when $t=1$, this is obviously ture. Suppose that it also holds at time $t-1$. Therefore, the policy selected at time step $t$ is not $i$, we will have $N_i(t) = N_i(t-1)$, and the above inequality stil holds. If the policy selected at time step $t$ is $i_{th}$ policy, then, this implies 
\begin{align}
    \widehat{V}_{i, N_i(t-1)} + \sqrt{\frac{a}{N_i(t-1)}} \geq \widehat{V}_{i^\star, N_{i^\star}(t-1)} + \sqrt{\frac{a}{N_{i^\star}(t-1)}} 
\end{align}
In addition, we have the event $E\paren{\state}$ is true. We further have
\begin{subequations}
    \begin{align}
        &\widehat{V}_{i^\star, N_{i^\star}(t-1)} + \sqrt{\frac{a}{N_{i^\star}(t-1)}}  \geq \VFunc_{i^\star}\paren{\state},\\
        & \widehat{V}_{i, N_i(t-1)} +\sqrt{\frac{a}{N_i(t-1)}} \leq \VFunc_{i}\paren{\state} + \frac{6}{5}\sqrt{\frac{a}{N_i(t-1)}}, \\
        \Rightarrow \quad & \frac{6}{5}\sqrt{\frac{a}{N_i(t-1)}}\geq \Delta_i.
    \end{align}
\end{subequations}
Since $N_i(t) = N_{i}(t-1)$, we get
\begin{align}
   N_i(t) \leq \frac{36}{25}\frac{a}{\Delta_i^2} + 1. 
\end{align}

Then, we prove the following,
\begin{align}
    N_i(t) \geq \frac{4}{25} \min \left\{\frac{a}{\Delta_i^2}, \frac{25}{36}(N_{i^\star}(t) - 1)\right\},\;\; \forall i\neq i^\star.
\end{align}
Following the same inductive proof, we only need to show the above is true if the policy $i^\star$ is selected at iteration $t$. Since the event $E\paren{\state}$ is true, we further have
\begin{align}
    \VFunc_{i^\star}\paren{\state}  + \frac{6}{5}\sqrt{\frac{a}{N_{i^\star}(t-1)}} \geq \VFunc_i\paren{\state} + \frac{4}{5}\sqrt{\frac{a}{N_i(t-1)}}
\end{align}
which gives us
\begin{subequations}
    \begin{align}
        &N_i(t-1) \geq \frac{16}{25} \frac{a}{\paren{\Delta_i + \frac{6}{5}\sqrt{\frac{a}{N_{i^\star}(t-1)}}}^2} \geq \frac{16}{25} \frac{a}{4\max\left\{\Delta_i ,\frac{6}{5}\sqrt{\frac{a}{N_{i^\star}(t-1)}}\right\}^2}.\\
        \Rightarrow\quad & N_i(t-1) \geq \frac{4}{25} \min\left\{\frac{a}{\Delta_i^2}, \frac{25}{36}(N_{i^\star}(t)-1)\right\}.
    \end{align}
\end{subequations}

Now, it remains to show for $i^\star$. We have
\begin{align}
    N_{i^\star}(T) - 1 = T- 1 - \sum_{i \neq i^\star} N_i(T) \geq T - \ONum - \frac{36}{25}a\sum_{i\neq i^\star} \Delta_i^2
\end{align}
By assumption, we have
\begin{align}
    a \leq \frac{25}{36}\frac{T-\ONum}{\sum_{i} \Delta_i^{-2}}
\end{align}
we get
\begin{align}
N_{i^\star}(T) - 1 \geq \frac{36}{25} a \Delta_{i^\star}^{-2},
\end{align}
Now, let 
\begin{align}
    2T\ONum \exp\paren{\frac{-2a}{50 \horizon^2}}= \delta.
\end{align}
We can conclude that, the optimal policy is identified with probability at least $1-\delta$, if %
\begin{align}
    T = \mathcal{O} \paren{\ONum + \paren{\sum_i \frac{\horizon^2}{\Delta_i^2}} \log \paren{\frac{\ONum}{\delta}}}
\end{align}
\end{proof}

\SampleComplexityMamba*
\begin{proof}
    We bound the following probability for iterations $T$,
    \begin{align}
         \Pr{{\widehat{\VFunc}_{i, T_i}(\state) - \widehat{\VFunc}_{i^\star, T_{i^\star}}(\state) \geq 0} }= \Pr{\widehat{\VFunc}_{i, T_i}(\state) - \widehat{\VFunc}_{i^\star, T_{i^\star}}(\state) - (-\Delta_i) \geq \Delta_i}
    \end{align}
    By Hoeffding's inequality, we have
    \begin{subequations}
        \begin{align}
             \Pr{\widehat{\VFunc}_{i, T_i}(\state) - \widehat{\VFunc}_{i^\star, T_{i^\star}}(\state) \geq 0} &\leq \exp \paren{\frac{-2 \paren{\lfloor T/\ONum \rfloor \Delta_i }^2}{2 \lfloor T/\ONum\rfloor \horizon^2}} \\
             &=\exp \paren{- \left\lfloor \frac{T}{\ONum} \right\rfloor  \Delta_i^2/\horizon^2}.
        \end{align}
    \end{subequations}
    By union bound, we have
    \begin{align}
         \Pr{ \max_{i\neq i^\star}\widehat{\VFunc}_{i, T_i)}(\state) - \widehat{\VFunc}_{i^\star, T_{i^\star}}(\state) < 0} \geq 1 - \sum_{i}\exp \paren{- \frac{T}{\ONum \horizon^2} \Delta_i^2}.
    \end{align}
    By solving the following equation, 
    \begin{align}
        \sum_{i}\exp \paren{- \frac{T}{\ONum \horizon^2} \Delta_i^2} = \delta. 
    \end{align}
    We get 
    \begin{align}
        T = \mathcal{O} \paren{ \paren{ \sum_{i}\frac{\ONum \horizon^2}{ \Delta_i^2}} \log \paren{\frac{\ONum}{\delta}}}.
    \end{align}
    
\end{proof}

\clearpage

%% file: supp_theory_proofs.tex
\clearpage

\subsection{Proof of \thmref{thm:ase:threshold}}

\begin{lemma}\label{lem:explore_bonus}
Given state $\state_t$, policy $\policy^k$, $\ENum_k\paren{\state_t}$ trajectories started from state $\state_t$, with probability $1-\delta$, we have
\[
\lvert {\VFunc}^{\policy^k}\paren{\state_t}-\widehat{\VFunc}^{\policy^k}\paren{\state_t}\rvert
\leq
\sqrt{\frac{2\horizon^2\log{\frac{2}{\delta}}}{\ENum_k\paren{\state_t}}}.\]
\end{lemma}

\begin{proof}[Proof of \lemref{lem:explore_bonus}]
    
\emph{Hoeffding's inequality} \citep{hoeffding63} implies for $X_1, \ldots ,X_n$ independent random variables, and W such that for all $i$, $\lvert X_i \rvert\leq W$ with probability $1$, and $\mu=\frac{1}{n}\sum X_i$, we have for all $t > 0$ that
\begin{equation}\label{eq:hpBound}
  \Pr{|\Bar{X}-\mu|>\frac{t}{\sqrt{n}}}\leq 2\cdot \exp\paren{-\frac{t^2}{2W^2}}
\end{equation}

Let $\delta=2 \exp\paren{-\frac{t^2}{2W^2}}$, we have

\begin{equation}
        t=\sqrt{2W^2\log\paren{\frac{2}{\delta}}}.\label{eq:hoeffding_t}
\end{equation}

By replacing $t$ in Equation~\ref{eq:hpBound}, we have
\begin{align*}
  \Pr{|\Bar{X}-\mu|>\sqrt{\frac{2W^2\log\paren{\frac{2}{\delta}}}{n}}}&\leq \delta\\
  \Pr{|\Bar{X}-\mu|\leq \sqrt{\frac{2W^2\log\paren{\frac{2}{\delta}}}{n}}}&\geq 1- \delta\\
  \Pr{|{\VFunc}^{\policy^k}\paren{\state_t}-\widehat{\VFunc}^{\policy^k}\paren{\state_t}|\leq \sqrt{\frac{2\horizon^2\log{\frac{2}{\delta}}}{\ENum_k\paren{\state_t}}}}&\stackrel{\paren{a}}{\geq} 1- \delta,\\
\end{align*}

where (a) follows by replacing $\Bar{X}=\widehat{\VFunc}^{\policy^k}\paren{\state_t}, \mu={\VFunc}^{\policy^k}\paren{\state_t},W=\horizon, n=\ENum_k\paren{\state_t}$.

So, with probability at least $1-\delta$, the following holds
\[
\lvert {\VFunc}^{\policy^k}\paren{\state_t}-\widehat{\VFunc}^{\policy^k}\paren{\state_t}\rvert \leq \sqrt{\frac{2\horizon^2\log{\frac{2}{\delta}}}{\ENum_k\paren{\state_t}}}.
\]

\end{proof}

\begin{proof}[Proof of \thmref{thm:ase:threshold}]
From \thmref{thm:APS-sample-complexity}, we have that for any state $\state$, the optimal policy is identified with probability at least $1-\delta$ if
\begin{align}\label{eq:nsbound}
    \ENum\paren{\state} = \mathcal{O}\left(K + \left(\sum_i \frac{\horizon^2}{\Delta_i^2}\right) \log\left(\frac{K}{\delta}\right)\right).
\end{align}
This means that, for state $\state$, after $\ENum\paren{\state} = \mathcal{O}\left(K + \left(\sum_i \frac{\horizon^2}{\Delta_i^2}\right) \log\left(\frac{K}{\delta}\right)\right)$ rounds state visitation, $\Func^{\textrm{max}}\paren{\state}$ will no long pick a suboptimal oracle with probability $1-\delta$.
Thus, any increment on $\ENum\paren{\state}$ will contribute to estimating the value of the best oracle $k_{\star}$. It will reduce the variance of the estimated value of oracle $k_{\star}$, which will 
lead to a decrease of the variance term $v$ in Equation~\ref{eq:regret-ub}.

{Now we would like to connect $\ENum\paren{\state}$ with threshold $\threshold$.} By \lemref{lem:explore_bonus} we know that we will require $\ENum_k\paren{\state_t}=\frac{2\horizon^2\log\frac{2}{\delta}}{\Gamma^2}$ samples to achieve uncertainty $\Gamma$ on state $\state$. Combining this result with Equation~\eqref{eq:nsbound}, we know that by setting
\begin{align*}
 \frac{2\horizon^2\log\frac{2}{\delta}}{\Gamma^2} = \mathcal{O}\paren{K + \left(\sum_i \frac{\horizon^2}{\Delta_i^2}\right) \log\left(\frac{K}{\delta}\right)}%
\end{align*}
and consequently, 
\begin{align*}
    \Gamma = o\paren{\sqrt{\frac{{2\horizon^2\log\frac{2}{\delta}}}{{K + \left(\sum_i \frac{\horizon^2}{\Delta_i^2}\right) \log\left(\frac{K}{\delta}\right)}}}},
\end{align*}
we can ensure selecting $k_{\star}$ with probability $1-2\delta$ (by the union bound). 

This is equivalent to stating that, with probability $1-\delta$, \lopsase identifies $k_{\star}$ in state $s$ with threshold
\begin{align}
    \Gamma = o\paren{\sqrt{\frac{{2\horizon^2\log\frac{4}{\delta}}}{{K + \left(\sum_i \frac{\horizon^2}{\Delta_i^2}\right) \log\left(\frac{2K}{\delta}\right)}}}},
\end{align}
hence completing the proof.
\end{proof}

%% file: supp_additional_experiments.tex
\section{Experiments}

\subsection{Baselines}\label{app:experiment:baselines}
\paragraph{PPO}
PPO uses a trust region optimization approach to update the policy. The trust region improves the stability of the learning process by ensuring the new policy is similar to the old policy. PPO also uses a clipped objective function for stability which limits the change in the policy at each iteration. %

\paragraph{AggreVaTeD}
AggreVaTeD is a differentiable version of AggreVaTe which focuses on a single oracle scenario. AggreVaTeD allows us to train policies with efficient gradient update procedures. AggreVaTeD uses a deep neural network to model the policy and use differentiable imitation learning to train it. By applying differentiable imitation learning, it minimize the difference between the expert's demonstration and the learner policy behavior. AggreVaTeD learn from the expert's demonstration while interact with the environment to outperform the expert.

\paragraph{PPO-GAE} GAE proposed a method to solving high-dimensional continuous control problems using reinforcement learning. GAE is based on a technique, \textit{generalized advantage estimation (GAE)}. GAE is used to estimate the advantage function for updating the policy. 
The advantage function measures how much better a particular action is compared to the average action. Estimating the advantage function with accuracy in high-dimensional continuous control problems is challenging. In this work, we propose PPO-GAE, which combines PPO's policy gradient method with GAE's advantage function estimate, which is based on a linear combination of value function estimates. By combining the advantage of PPO and GAE, PPO-GAE achieved both sample efficiency and stability in high-dimensional continuous control problems.

\paragraph{MAMBA}
MAMBA is the SOTA work of learning from multiple oracles. It utilizes a mixture of imitation learning and reinforcement learning to learn a policy that is able to imitate the behavior of multiple experts.  MAMBA is also considered as interactive imitation learning algorithm, it imitate the expert and interact with environment to improve the performance. MAMBA randomly select the state as switch point between learner policy and oracle. Then, it randomly select oracle to roll out. It effectively combines the strengths of multiple experts and able to handle the case of conflicting expert demonstrations.

\subsection{DeepMind Control Suite Environments}
We conduct an evaluation of \algname, comparing its performance to the mentioned baselines, on the Cheetah-run, CartPole-swingup, Pendulum-swingup, and Walker-walk tasks sourced (see \figref{fig:exp:env}) from the DeepMind Control Suite~\citep{tassa2018deepmind}.

\begin{figure}
    \centering
    \begin{subfigure}{.22\textwidth}
        \centering
        \includegraphics[width=\textwidth]{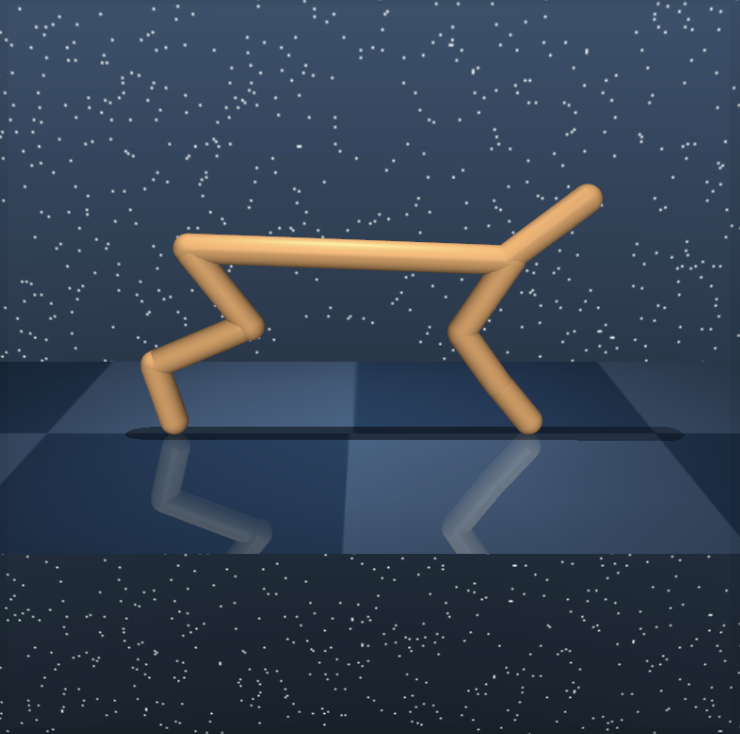}
        \caption{Cheetah-run}
    \end{subfigure}
    \begin{subfigure}{.22\textwidth}
        \centering
        \includegraphics[width=\textwidth]{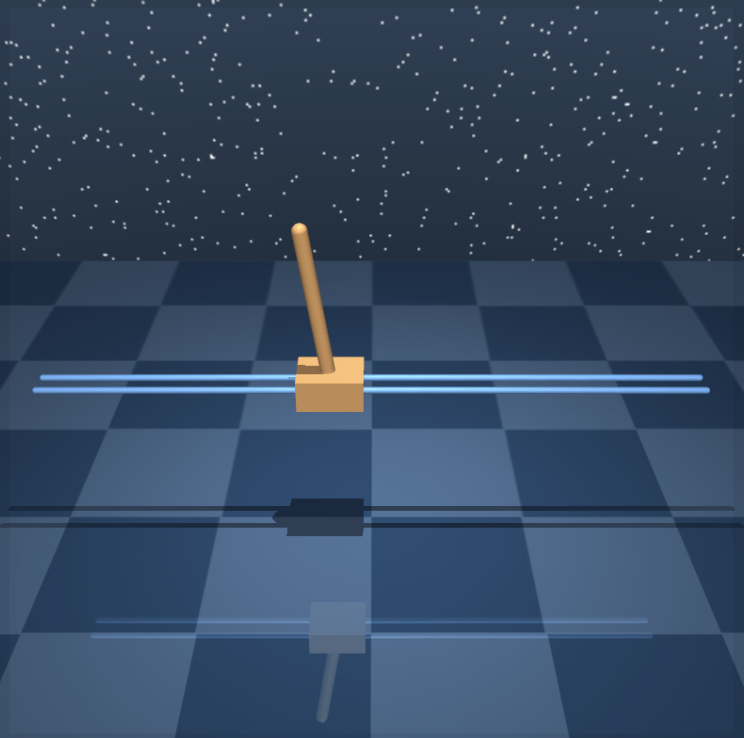}
        \caption{Carpole-swingup}
    \end{subfigure}
    \begin{subfigure}{.22\textwidth}
        \centering
        \includegraphics[width=\textwidth]{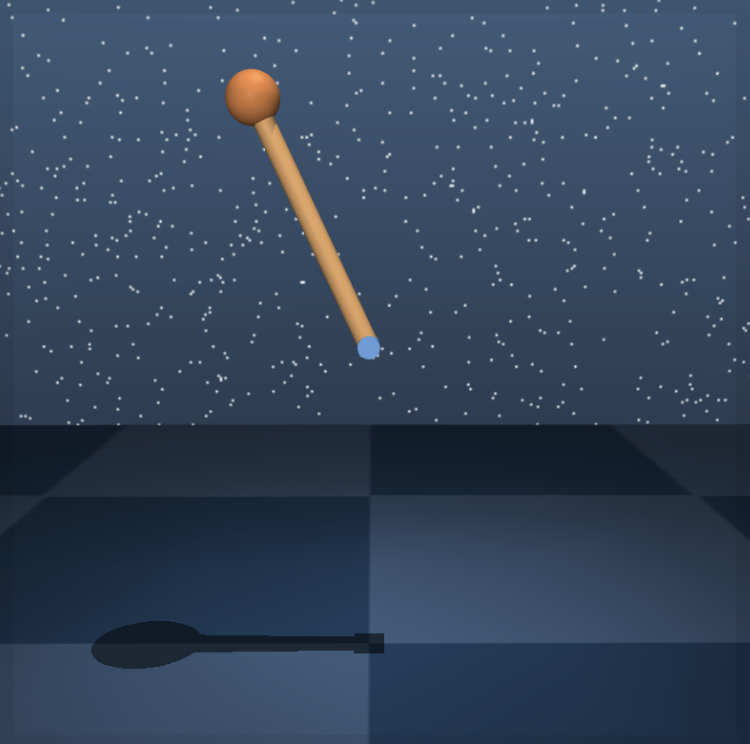}
        \caption{Pendulum-swingup}
    \end{subfigure}
    \begin{subfigure}{.216\textwidth}
        \centering
        \includegraphics[width=\textwidth]{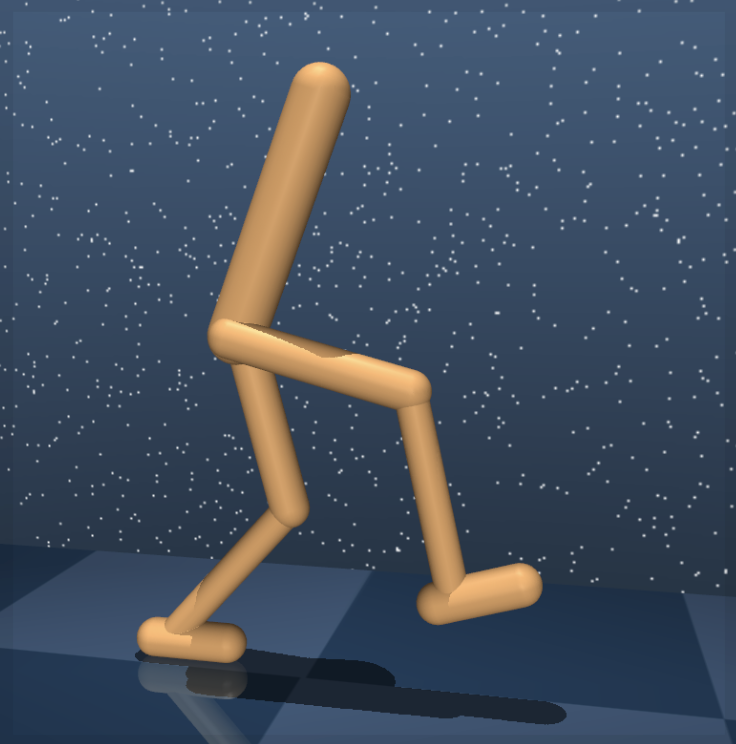}
        \caption{Walker-walk}
    \end{subfigure}
    \caption{A visual illustration of the tasks we adopted for our experiments.
    }\label{fig:exp:env}
\end{figure}

\subsection{Implementation Details of \lopsase}\label{app:algdetails}

We provide the details of $\algname$ in \fix{\algoref{alg:lops}} as \algoref{alg:lops:implement}. \algoref{alg:lops:implement} closely follows \fix{\algoref{alg:lops}} with a few modifications as follows:

\begin{itemize}
    \item In line \ref{line:buffer}, we have a buffer with a fixed size ($\|\mathcal{D}_n\| = 19,200$) for each oracle, and we discard the oldest data when it fills up.
    \item In line \ref{line:rollout}, we roll-out the learner policy until a buffer with a fixed size ($\|\mathcal{D}_n'\|$ = 2,048) fills up, and empty it once we use them to update the learner policy. This stabilizes the training compared to storing a fixed number of trajectories to the buffer, as MAMBA does.
    \item In line \ref{line:update}, we adopted PPO style policy update.
\end{itemize}

\begin{algorithm}[!t]
    \caption{\textbf{M}ax-aggregation \textcolor{Maroon}{\textbf{A}ctive \textbf{P}olicy \textbf{S}election} with \textcolor{NavyBlue}{Active \textbf{S}tate \textbf{E}xploration} (M\textcolor{Maroon}{APS}-\textcolor{NavyBlue}{SE})}\label{alg:lops:implement}
    \label{alg:lops:implement2}
    \begin{algorithmic}[1] 
        \Require {Initial learner policy $\policy_{1}$}, oracle policies $\{\policy^{k}\}_{k\in \bracket{\ONum}}$, initial value functions $\{\hat{V}^k\}_{k\in\bracket{\ONum}}$
        \For{$n=1,2, \ldots, N-1$} 
        \If{\textcolor{NavyBlue}{SE is $\textsc{True}$}}
            \IndentLineComment{/* active state exploration */}
            \State \textcolor{NavyBlue}{Roll-in learner policy $\policy_n$ until $\Gamma_{k_\star}\paren{\state_t}$ ${\geq} \threshold$, where $k_{\star}$ and $\Gamma_{k_\star}\paren{\state_{t}}$ are computed via Equation~\ref{eq:kstar}  and \ref{eq:bonus} at each visited state $s_t$.}
        \Else
            \State Roll-in learner policy $\policy_n$ up to $t_e\sim \text{Uniform}\bracket{\horizon-1}$
        \EndIf
        \IndentLineComment{/* active policy selection */}
        \State \textcolor{Maroon}{{Select $k_{\star}$ via Equation~\ref{eq:kstar}}}.
\State {Switch to $\policy^{k_{\star}}$ to roll-out and collect data $\mathcal{D}_n$}. We have a buffer with a fixed size ($\|\mathcal{D}_n\| = 19,200$) for each oracle, and we discard the oldest data when it fills up.\label{line:buffer}
\State Update the estimate of $\widehat{V}^{k_{\star}}(\cdot)$ with $\mathcal{D}_n$.
\State Roll-out the learner policy until a buffer with a fixed size ($\|\mathcal{D}_n'\|$ = 2,048) fills up, and empty it once we use them to update the learner policy. This stabilizes the training compared to storing a fixed number of trajectories to the buffer \label{line:rollout}
\State Compute gradient estimator $g_n$ of $\nabla \widehat{\ell}_n(\policy_n, \lambda)$ in \eqref{eq:gradient} using $\mathcal{D}_n'$.
\State Perform PPO style policy update on policy $\policy_n$ to $\policy_{n+1}$.\label{line:update}
        
        \EndFor 
    \end{algorithmic}
\end{algorithm}

\subsection{Computing Infrastructure}
 We performed our experiments on a cluster that includes CPU nodes (about 280 cores) and GPU nodes, about 110 Nvidia GPUs, ranging from Titan X to A6000, set up mostly in 4- and 8-GPU nodes.